\RequirePackage[l2tabu,orthodox]{nag}
\documentclass
[letterpaper,12pt,]
{article}

\usepackage{etex}
\usepackage{verbatim}
\usepackage{comment}
\usepackage{xspace,enumerate}
\usepackage[dvipsnames]{xcolor}
\usepackage[T1]{fontenc}
\usepackage[full]{textcomp}
\usepackage[american]{babel}
\usepackage{mathtools}
\usepackage{amsthm}
\usepackage[
letterpaper,
top=1in,
bottom=1in,
left=1in,
right=1in]{geometry}
\usepackage{newpxtext} 
\usepackage{textcomp} 
\usepackage[varg,bigdelims]{newpxmath}
\usepackage[scr=rsfso]{mathalfa}
\usepackage{bm} 
\let\mathbb\varmathbb
\usepackage{microtype}
\usepackage[pagebackref=true,colorlinks=true,urlcolor=blue,linkcolor=blue,citecolor=OliveGreen]{hyperref}
\usepackage[capitalise,nameinlink]{cleveref}
\crefname{lemma}{Lemma}{Lemmas}
\crefname{fact}{Fact}{Facts}
\crefname{theorem}{Theorem}{Theorems}
\crefname{corollary}{Corollary}{Corollaries}
\crefname{claim}{Claim}{Claims}
\crefname{example}{Example}{Examples}
\crefname{algorithm}{Algorithm}{Algorithms}
\crefname{problem}{Problem}{Problems}
\crefname{definition}{Definition}{Definitions}
\crefname{exercise}{Exercise}{Exercises}
\crefname{model}{Model}{Models}
\usepackage{amsthm}

\newtheorem{theorem}{Theorem}[section]
\newtheorem*{theorem*}{Theorem}
\newtheorem{lemma}[theorem]{Lemma}
\newtheorem*{lemma*}{Lemma}

\newtheorem*{fact*}{Fact}

\newtheorem*{proposition*}{Proposition}
\newtheorem{corollary}[theorem]{Corollary}
\newtheorem*{corollary*}{Corollary}

\newtheorem*{hypothesis*}{Hypothesis}

\newtheorem*{conjecture*}{Conjecture}
\theoremstyle{definition}
\newtheorem{definition}[theorem]{Definition}
\newtheorem*{definition*}{Definition}

\newtheorem*{construction*}{Construction}

\newtheorem*{example*}{Example}

\newtheorem*{question*}{Question}
\newtheorem{algorithm}[theorem]{Algorithm}
\newtheorem*{algorithm*}{Algorithm}

\newtheorem*{assumption*}{Assumption}

\newtheorem*{problem*}{Problem}

\newtheorem*{openquestion*}{Open Question}
\theoremstyle{remark}
\newtheorem{claim}[theorem]{Claim}
\newtheorem*{claim*}{Claim}

\newtheorem*{remark*}{Remark}

\newtheorem*{observation*}{Observation}
\theoremstyle{model}

\newtheorem*{model*}{Model}
\usepackage{paralist}
\let\originalleft\left
\let\originalright\right
\renewcommand{\left}{\mathopen{}\mathclose\bgroup\originalleft}
\renewcommand{\right}{\aftergroup\egroup\originalright}
\usepackage{turnstile}
\usepackage{mdframed}
\usepackage{tikz}
\usetikzlibrary{positioning}
\usetikzlibrary{decorations.pathreplacing}
\usepackage{caption}
\DeclareCaptionType{Algorithm}
\usepackage{newfloat}
\usepackage{array}
\usepackage{subfig}
\usepackage{xparse}
\usepackage{amsthm} 
\makeatletter
\let\latexparagraph\paragraph
\RenewDocumentCommand{\paragraph}{som}{%
  \IfBooleanTF{#1}
    {\latexparagraph*{#3}}
    {\IfNoValueTF{#2}
       {\latexparagraph{\maybe@addperiod{#3}}}
       {\latexparagraph[#2]{\maybe@addperiod{#3}}}%
  }%
}
\newcommand{\maybe@addperiod}[1]{%
  #1\@addpunct{.}%
}
\makeatother

\usepackage{boxedminipage}
\newenvironment{algorithmbox}{\begin{mdframed}[nobreak=true]
\begin{algorithm}}{\end{algorithm}\end{mdframed}}

\newcommand{\Paren}[1]{\left(#1\right)}
\newcommand{\bigparen}[1]{\big(#1\big)}
\newcommand{\Bigparen}[1]{\Big(#1\Big)}

\newcommand{\Brac}[1]{\left[#1\right]}

\newcommand{\Set}[1]{\left\{#1\right\}}

\newcommand{\norm}[1]{\lVert#1\rVert}
\newcommand{\Norm}[1]{\left\lVert#1\right\rVert}

\newcommand{\iprod}[1]{\langle#1\rangle}
\newcommand{\Iprod}[1]{\left\langle#1\right\rangle}

\newcommand{\Esymb}{\mathbb{E}}
\newcommand{\Psymb}{\mathbb{P}}

\DeclareMathOperator*{\E}{\Esymb}

\DeclareMathOperator*{\ProbOp}{\Psymb}
\renewcommand{\Pr}{\ProbOp}

\newcommand\bdot\bullet

\DeclareMathOperator{\Ind}{\mathbf 1}

\DeclareMathOperator{\Tr}{Tr}
\DeclareMathOperator{\SDP}{SDP}

\DeclareMathOperator{\poly}{poly}

\DeclareMathOperator{\polylog}{polylog}

\newcommand{\Erdos}{Erd\H{o}s\xspace}
\newcommand{\Renyi}{R\'enyi\xspace}

\newcommand{\Z}{\mathbb Z}

\newcommand{\R}{\mathbb R}

\newcommand{\ddelta}{d_\delta}
\newcommand{\mudelta}{\mu_\delta}

\newcommand{\Deltadelta}{\Delta_\delta}
\renewcommand{\leq}{\leqslant}

\renewcommand{\geq}{\geqslant}

\let\epsilon=\varepsilon
\numberwithin{equation}{section}
\newcommand\MYcurrentlabel{xxx}
\newcommand{\MYstore}[2]{%
  \global\expandafter \def \csname MYMEMORY #1 \endcsname{#2}%
}
\newcommand{\MYload}[1]{%
  \csname MYMEMORY #1 \endcsname%
}
\newcommand{\MYnewlabel}[1]{%
  \renewcommand\MYcurrentlabel{#1}%
  \MYoldlabel{#1}%
}
\newcommand{\MYdummylabel}[1]{}
\newcommand{\torestate}[1]{%
  \let\MYoldlabel\label%
  \let\label\MYnewlabel%
  #1%
  \MYstore{\MYcurrentlabel}{#1}%
  \let\label\MYoldlabel%
}
\newcommand{\restatetheorem}[1]{%
  \let\MYoldlabel\label
  \let\label\MYdummylabel
  \begin{theorem*}[Restatement of \cref{#1}]
    \MYload{#1}
  \end{theorem*}
  \let\label\MYoldlabel
}
\newcommand{\restatelemma}[1]{%
  \let\MYoldlabel\label
  \let\label\MYdummylabel
  \begin{lemma*}[Restatement of \cref{#1}]
    \MYload{#1}
  \end{lemma*}
  \let\label\MYoldlabel
}
\newcommand{\restateprop}[1]{%
  \let\MYoldlabel\label
  \let\label\MYdummylabel
  \begin{proposition*}[Restatement of \cref{#1}]
    \MYload{#1}
  \end{proposition*}
  \let\label\MYoldlabel
}
\newcommand{\restatefact}[1]{%
  \let\MYoldlabel\label
  \let\label\MYdummylabel
  \begin{fact*}[Restatement of \cref{#1}]
    \MYload{#1}
  \end{fact*}
  \let\label\MYoldlabel
}
\newcommand{\restate}[1]{%
  \let\MYoldlabel\label
  \let\label\MYdummylabel
  \MYload{#1}
  \let\label\MYoldlabel
}

\newcommand{\eps}{\epsilon}

\allowdisplaybreaks
\newcommand*{\Id}{\mathrm{Id}}

\newcommand*{\normop}[1]{\norm{#1}_{\mathrm{op}}}
\newcommand*{\Normop}[1]{\Norm{#1}_{\mathrm{op}}}

\newcommand*{\Normf}[1]{\Norm{#1}_{\mathrm{F}}}

\newcommand{\sbm}{\text{sbm}_{d,\eps}}
\newcommand{\SBM}{\mathsf{SBM}}

\DeclareMathOperator{\diag}{diag}

\newcommand*{\dyad}[1]{#1#1{}^{\mkern-1.5mu\mathsf{T}}}

\title{Reaching Kesten-Stigum Threshold in the Stochastic Block Model under Node Corruptions  \thanks{This project has received funding from the European Research Council (ERC) under the European Union's Horizon 2020 research and innovation programme (grant agreement No 815464).}}
\usepackage{times}

\author{
Jingqiu Ding\thanks{ETH Z\"urich.}
\and
Tommaso d'Orsi\footnotemark[2]
\and
Yiding Hua\footnotemark[2]
\and
David Steurer\footnotemark[2]
}

\begin{document}

\pagestyle{empty}


\maketitle
\thispagestyle{empty} 


\begin{abstract}

We study robust community detection in the context of node-corrupted stochastic block model, 
where an adversary can arbitrarily modify all the edges incident to a fraction of the $n$ vertices. 
We present the first polynomial-time algorithm that achieves weak recovery at the Kesten-Stigum threshold even in the presence of a small constant fraction of corrupted nodes.
Prior to this work, even state-of-the-art robust algorithms were known to break under such node corruption adversaries, when close to the Kesten-Stigum threshold.

We further extend our techniques to the $\Z_2$ synchronization problem, 
where our algorithm reaches the optimal recovery threshold in the presence of similar strong adversarial perturbations.

The key ingredient of our algorithm is a novel identifiability proof that leverages the push-out effect of the Grothendieck norm of principal submatrices.

\end{abstract}

\newpage



\microtypesetup{protrusion=false}
\tableofcontents{}
\microtypesetup{protrusion=true}

\clearpage

\pagestyle{plain}
\setcounter{page}{1}


\section{Introduction}
Community detection is the problem of identifying hidden communities in random graphs that are generated based on some planted structures. 
The stochastic block models are a family of random graph models that, for a variety of reasons, play a central role in the study of community detection 
(see the excellent survey of \cite{AbbeReview}). 
In this work, we focus on balanced two-community stochastic block model, which is defined as follows:

\begin{definition}[Balanced two-community stochastic block model]
  \label{def:SBM}
  For parameters $\epsilon\in (0,1),n>0,d>0$, for some label vector $x^* \in \{ \pm 1 \}^n$ such that $\mathbf{1}^\top x^*=0$,
  the balanced two-community stochastic block model $\SBM_{d,\epsilon}(x^*)$ describes the following distribution over graph with $n$ vertices: 
  every pair of distinct vertices $i,j\in [n]$ in the graph is connected by an edge independently with probability $(1+\epsilon x_i^*x_j^*)\frac{d}{n}$.
\end{definition}

In the above definition, $\epsilon$ is the bias parameter, $d$ is the average degree and each vertex $i$ gets a label $x^*_i$. Given a graph G sampled according to this model, the goal is to recover the (unknown) underlying vector of labels $x^*$ as accurate as possible.

\begin{definition}[Weak recovery]
For $G \sim \sbm$, an estimator $\hat{x}(G)\in \{\pm 1\}^n$ is said to achieve weak recovery if and only if $\E\iprod{\hat{x}(G), x^*}^2 \geq \Omega(n^2)$.
\end{definition}

In the last decade, a celebrated line of  works established that weak recovery can be achieved (and efficiently so) if and only if $\epsilon^2 d>1$ when $d = o(n)$
(\cite{decelle2011asymptotic,massoulie2014community,mossel2014belief,mossel2015reconstruction,DenseSBMThreshold}). 
This threshold  is called the \textit{Kesten-Stigum threshold} (henceforth KS threshold).
The algorithms introduced in these works, however, are brittle and can be fooled by very small adversarial perturbations.\footnote{In the sense that, altering $n^{o(1)}$ edges is enough to break them.}
Later works (\cite{montanari2016semidefinite, ding2022robust}) showed that, perhaps surprisingly, we can achieve weak recovery at the KS threshold, even in the presence of edge perturbations.

It is easy to see that weak recovery becomes information theoretically impossible when more than $\eps\cdot d\cdot n$ edges are corrupted.\footnote{An adversary may remove each intra-cluster edge with probability $\eps\cdot d/n$ and add each inter-cluster edge with probability $\eps\cdot d/n$. With high probability such process alter at most $\epsilon\cdot d\cdot n(1+o(1))$ edges. The graph now is indistinguishable from an \Erdos-\Renyi graph.}
In light of this, \cite{liu2022minimax} explored an alternative model of corruption, where a constant fraction of vertices have all of their incident edges modified.

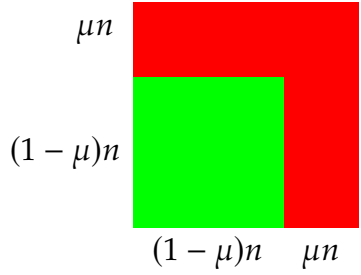
\begin{figure}
  \begin{center}
  \begin{tikzpicture}
    \fill[red] (1, 0.7) rectangle (4, 3.7);
    \fill[green] (1, 0.7) rectangle (3, 2.7);
    \draw (1,1.7) node[anchor=east]{$(1-\mu)n$};
    \draw (2,0) node[anchor=south]{$(1-\mu) n$};
    \draw (3.5,0) node[anchor=south]{$\mu n$};
    \draw (0.5,3) node[anchor=south]{$\mu n$};
  \end{tikzpicture}
  \end{center}
  \caption{For the adjacency matrix generated from \cref{def:node_corrupted_SBM},
  the entries in the red area are adversarially
  corrupted, while the entries in the green area 
  are sampled as in stochastic block model.}\label{fig:node-corruption}
\end{figure}

\begin{definition}[Node-corrupted balanced two-community stochastic block model]
\label{def:node_corrupted_SBM}
Given $\mu \in [0, 1)$ and $G^0\sim \sbm$,
an adversary may choose up to $\mu n$ vertices in $G^0$ and arbitrarily modify incident edges of them to produce the corrupted graph $G$. 
\end{definition}

In this node corruption model, the adversary is allowed to change an arbitrary number of edges for each corrupted vertex and could introduce up to $O(\mu n^2)$ edges.
While vertices of untypically large degree are algorithmically easy to identify and remove, less naive adversaries may remove all edges of an arbitrary subset of $\mu n$ vertices, and replace them by roughly $d$ spurious edges of their choosing. Such an adversary would introduce $O(\mu\cdot d\cdot n)$ edges, causing the algorithms of \cite{montanari2016semidefinite, ding2022robust} to fail when $\mu \geq \eps$.

As a result, node corruption settings present new algorithmic challenges that previous approaches fail to overcome. 
\cite{liu2022minimax} developed a polynomial-time algorithm which achieves minimax error rates when $\epsilon^2 d>C$ and the corruption ratio is $\mu=o(1)$ (where $C$ is a large universal constant). 
\footnote{The algorithm in \cite{liu2022minimax} is further robust against a different type of adversary corruptions with unbounded monotone changes from semi-random model defined in \cite{Moitra16}.}

However, their algorithm cannot achieve weak recovery when the signal-to-noise ratio is close to the KS threshold, i.e. when $\epsilon^2 d=1+\delta$ for small $\delta$.

In this work, we develop algorithms that can exploit the special structure of node corruptions and \textit{achieve node-robust weak recovery} as long as $\mu \leq \mudelta$,  where $\mudelta$ is a constant that only depends on $\delta$ but not on $\epsilon$ or $d$.

\subsection{Results}

Our main result is an efficient algorithm to achieve robust weak recovery under node corruptions.

\begin{theorem}\label{thm:main}
Let $n>1,d > 1, \eps \in (0,1)$ and balanced label vector $x^*\in \Set{\pm 1}^n$. 
Let $\delta\coloneqq \eps^2 d-1$. 
Let $G$ be a corrupted graph generated from \cref{def:node_corrupted_SBM} given $G^0\sim \SBM_{d,\epsilon}(x^*)$ and $\mu \in [0, 1)$.
When $\delta\geq \Omega(1)$ and $\mu\leq \Omega_\delta(1)$ \footnote{$\mu\leq \Omega_\delta(1)$ here means that $\mu$ is bounded by a constant depending on $\delta$. 
The dependence on $\delta$ is necessary: if $\mu$ is a fixed constant, then the recovery impossible for a small enough constant $\delta$ (see \cref{sec:lower-bound} for details).}, there exists a polynomial-time algorithm that outputs a labelling $\hat{x}\in \Set{\pm 1}^n$ such that
\begin{equation*}  
    \E[\iprod{\hat{x}, x^*}^2] \geq \Omega(n^2)
\end{equation*}
\end{theorem}

Our algorithm is the first one that succeeds in this setting, \cite{liu2022minimax} cannot work unless $\delta$ is sufficiently large, and \cite{montanari2016semidefinite,ding2022robust} cannot tolerate the corruption of $\Omega_\delta(1)$ vertices. 
Additionally, our algorithm is optimal in the sense that for $\epsilon^2 d< 1$, it is information theoretically impossible to achieve weak recovery (see \cite{mossel2015reconstruction}).

Based on the techniques we use for the robust stochastic block model, we also give an algorithm for the closely related robust $\Z_2$ synchronization problem, which can be formulated as follows:
\begin{definition}[Row/column-corrupted $\Z_2$ Synchronization model]\label{def:Z2-corruption}
	Given a hidden vector $x^* \in \{\pm 1\}^{n}$ and $\sigma > 0$, let $A^0$ be the uncorrupted $\Z_2$ synchronization matrix
	\begin{equation*}
		A^0 = \sigma x^* (x^*)^{\top} + W
	\end{equation*}
	where $W \in \R^{n\times n}$ is a symmetric random matrix 
	whose upper triangular entries are i.i.d sampled from $N(0, n)$. 
	An adversary may select $\mu n$ elements of $[n]$ and arbitrarily modify the corresponding rows and 
	columns of $A^0$ to produce a corrupted matrix $A$ that we observe.
\end{definition}
When $\sigma\leq 1$, even with no corruptions (i.e. $\mu=0$), it is information 
theoretically impossible to achieve weak recovery (\cite{LowerBoundPCA}). 
When $\sigma \geq 1+\Omega(1)$ and $\mu=0$, a polynomial-time algorithm is known to 
output estimator $\hat{x}\in \Set{\pm 1}^n$ such that $\iprod{\hat{x},x}^2\geq \Omega(n^2)$ with high probability. 
This is due to the thresholding phenomenon called the BBP transition (\cite{BBPtransition}).
However, for reasons similar to those described in the SBM settings, when $\mu\geq \Omega(1)$, the analysis of known algorithms such as semidefinite programming (\cite{montanari2016semidefinite}) or spectral algorithm (\cite{LowerBoundPCA}) breaks down.
In this paper, we give an algorithm that can achieve the constant sharp threshold for robust $\Z_2$ synchronization:
\begin{theorem}[Proved in \cref{sec:robust_Z2}]\label{thm:sos-z2-informal}
Given a row/column corrupted matrix $A$ generated from \cref{def:Z2-corruption}, 
when $\sigma \geq 1+\Omega(1)$, there is a polynomial-time algorithm that outputs an estimator $\hat{x}\in \{\pm 1\}^n$ such that $\E\Brac{\iprod{\hat{x},x^*}^2}\geq \Omega(n^2)$ with high probability over $x^*$ and $A^0$. 
\end{theorem}

\section{Related previous work}

\paragraph{Edge corruption robust algorithms}
A large body of work, concerning stochastic block model, has focused on the settings where an adversary may arbitrarily modify $\Omega(n)$  edges.
When the average degree $d$ diverges, \cite{montanari2016semidefinite} showed that a simple semidefnite program is robust to such perturbations.
 \cite{ding2022robust} presented a complementary result, providing weak recovery guarantees for the case when the average degree is bounded.

\paragraph{Node corruption robust algorithms}
Constant-fraction node corruptions were first studied in \cite{liu2022minimax}, where the authors obtained a polynomial-time algorithm that achieves optimal recovery rates when $\epsilon^2 d - 1$ is a sufficiently large constant. Although this provides weak recovery for large values of $\epsilon^2 d$, it falls short of reaching the KS threshold, where $\epsilon^2 d$ approaches $1$. In contrast, \cite{stephan2019robustness} provides a fast algorithm that achieves the KS threshold, but their algorithm is only robust to the corruption of $O(n^{0.001})$ vertices in sparse graphs, offering weaker robustness guarantees compared to our paper.

\paragraph{Other notions of robustness}
The algorithm of \cite{liu2022minimax} is further robust against unbounded monotone changes. \textit{Monotone adversaries} have been shown to make it information theoretically impossible to reach the KS threshold (\cite{Moitra16}).
This implies that finding algorithms robust against monotone adversaries is infeasible in our settings. 

\cite{abbe2020graph} proposed a spectral powering algorithm that is able to practically achieve the KS threshold. This algorithm is robust against tangles and cliques, making it applicable to the geometric block model. 
However, it is not expected to provide comparable robustness guarantees under stronger adversaries.

\paragraph{$\Z_2$ synchronization}
In the related $\Z_2$ synchronization model, similar guarantees have also been obtained. When the number of corrupted entries is limited to $\tilde{O}(n^{3/2})$, the weak recovery threshold $\sigma>1$ (BBP threshold) can be achieved via basic semidefinite programming (\cite{montanari2016semidefinite}). Recently, \cite{liu2022minimax} showed that, under the row/column-corruption model (see \cref{def:Z2-corruption}), achieving the minimax error rate is possible with $o(n)$ corrupted rows/columns, when $\sigma$ is greater than a sufficiently large constant. However, when $\sigma$ is close to $1$ and $\Omega(n)$ rows/columns are corrupted, no weak recovery algorithms (including inefficient algorithms) are known.

\section{Techniques}
We outline here the main ideas behind \cref{thm:main} and \cref{thm:sos-z2-informal}.
The algorithm is splitted into two components, each tailored to handle one degree regime.
For the regime with average degree $d > \ddelta$, where $\ddelta$ is a constant that only depends on $\delta:=d\eps^2-1$, our starting point is the result of \cite{montanari2016semidefinite}. 
For the sparse regime with $d \leq \ddelta$, we will borrow from \cite{ding2022robust}.

\paragraph{Push-out effect of the basic SDP}
Consider the settings $d > \ddelta$ and, for simplicity of the exposition, assume $d=\omega(1)$ and $\mu=o(1)$.  \cite{montanari2016semidefinite} proved that the following SDP program --which we refer to as the basic SDP--
\begin{equation*}
	\SDP(M)
	= \max \left\{ \Iprod{M, X} : X \succeq 0, X_{ii} = 1  \forall i \in [n] \right\}
\end{equation*}
achieves weak recovery at the KS threshold. Concretely, for an uncorrupted graph $G^0 \sim \sbm$ with centered adjacency matrix $\tilde{A^0}$, they showed for some constant $\Deltadelta>0\,,$
\begin{align}
	 \SDP(\tilde{A^0})
	&\geq (2 + \Deltadelta) n \sqrt{d}\,,\label{eq:basic-sdp-push-out}\\
	\SDP(\tilde{A^0} - \frac{\epsilon d}{n} \dyad{x^*})
	&\leq \Paren{2 + \frac{\Deltadelta}{2}} n \sqrt{d}\,.\label{eq:basic-sdp-push-in}
\end{align}
with probability $1 - \exp(-\Omega_{\delta}(n))$.
Taken together, these inequalities highlight a significant shift in the SDP value resulting from the subtraction of a rank-$1$ matrix from $\tilde{A^0}$. This phenomenon is often referred to as the \textit{push-out effect} and has often been exploited to design algorithms (\cite{LowerBoundPCA,montanari2016semidefinite,ding2022robust}). 
Additionally, the exponential concentration probability allows us to demonstrate that the \textit{push-out effect} occurs for every principal submatrix of size 
$(1-o(1))n\times (1-o(1))n$

It is easy to see that the basic SDP is robust against \textit{some} adversarial perturbations. A single edge alteration can change both \cref{eq:basic-sdp-push-out} and \cref{eq:basic-sdp-push-in} by at most $2$.
As $ \SDP(\tilde{A^0}) - \SDP(\tilde{A^0} - \frac{\epsilon d}{n} \dyad{x^*}) \geq \Omega(n\sqrt{d})\,,$ as long as the number of edge corruptions is bounded from above by $O(n\sqrt{d})$, the algorithm can still approximately recover the communities.

While this algorithm is robust to some edge adversarial perturbations, it is highly non-trivial whether it still works in presence of $\Omega(1)$-fraction of corrupted nodes, even if we assume all corrupted nodes have degree $O(d)$. 
In the node corruption model, the number of modified edges can reach $\Omega(n \cdot d)$, which is far more than $n\sqrt{d}$ when $d=\omega(1)$. 
The gap between $n d$ and $n\sqrt{d}$ indicates that we need a fundamentally different approach to find algorithms robust to $\Omega(n)$ corrupted nodes.

\paragraph{Push-out effect of submatrices}
A priori it is not clear whether it is possible to recover the signal in presence of node corruptions, or if such an adversary has the capability of hiding all the information. 
A good news is that, while the basic SDP is fragile to node corruptions, it \textit{suggests} a plausible direction to design an (inefficient!) algorithm robust to node corruptions. The key observation is that there is always a principal submatrix of size $(1-\mu)n\times (1-\mu)n$ free from corruption. More specifically, let $\tilde{A}$ be the adjacency matrix of the corrupted graph, the structure of node corruptions implies that the uncorrupted vertices $S^*\subseteq [n]$ satisfies $\tilde{A}_{S^*}=\tilde{A}^0_{S^*}$, where $\tilde{A}_{S^*}$ and $\tilde{A}^0_{S^*}$ denote the submatrix of $\tilde{A}$ and $\tilde{A}^0$ restricted to the set $S^* \times S^*$. Moreover, it can be shown that, with high probability, the push-out effect \textit{still holds} for this submatrix. That is:
\begin{align}
	\SDP(\tilde{A}_{S^*})
	&\geq (2 + \Deltadelta) (1-\mu) n \sqrt{d}\,,\label{eq:submatrix-basic-sdp-push-out}\\
	\SDP\Paren{\tilde{A}_{S^*} - \frac{\epsilon d}{n} x_{S^*}^*x^{*\top}_{S^*}}
	&\leq \Paren{2 + \frac{\Deltadelta}{2}} (1-\mu) n \sqrt{d}\,.\label{eq:submatrix-basic-sdp-push-in}
\end{align}
In other words, if we \textit{knew} the set of uncorrupted nodes, then we would still be able to approximately recover the communities.

Unfortunately, the set of uncorrupted nodes $S^*$ is not immediately known. Moreover, even disregarding computational issues, it remains unclear how one could identify such a set.
A rudimentary strategy to address this challenge would be to identify a subset $S\subseteq [n]$ such that the objective value of the $\SDP(\tilde{A}_{S})$ is large and to use the optimizer $X$ as an estimator. However, this approach presents a problem in that the selected set $S$ may contain corrupted vertices, leading to a situation where the optimizer $X$ may align with the corruption rather than accurately reflecting the true labels.\footnote{Note that even with no corruption, when $\delta$ is a small constant, the optimizer $X$ is only weakly correlated with $\dyad{x^*}$.}

A natural way to circumvent this issue, is to search over \textit{pairs} $(S, X)$ where $S\subseteq [n]$ and $X$ is a positive semidefinite matrix that fulfills the \textit{submatrix push-out constraints} that is described below.

\begin{definition}\label{def:submatrix-pushout}
	Given a corrupted graph $G$ as described in \cref{def:node_corrupted_SBM} and its centered adjacency matrix $\tilde{A}$, consider a set $S\subseteq [n]$ such that $|S|=(1-\mu) n$ and a positive semidefinite matrix $X$ where $X_{ii}=1$ for all $i\in [n]$,
	we say that the triplet $(\tilde{A},S,X)$ satisfies \textit{submatrix push-out constraints} if and only if for every subset $S'\subseteq S$ such that $|S'|\geq (1-2\mu)n$, it holds that
	\begin{equation*}
		\iprod{\tilde{A}_{S'},X_{S'}} \geq (2+\Deltadelta) (1- o(1)) n\sqrt{d}\,.
	\end{equation*}
\end{definition}

The \textit{submatrix push-out constraints} is useful because, if one can find $S,X$ such that $(\tilde{A}, S,X)$ satisfies the \textit{submatrix push-out constraints}, then for $S'=S \cap S^*$, which is the set of uncorrupted nodes in set $S$ and has size at least $(1-2\mu)n$, it follows from \cref{def:submatrix-pushout} that

\begin{equation*}
	\iprod{\tilde{A}_{S'},X_{S'}} \geq (2+\Deltadelta) (1- o(1)) n\sqrt{d} \,.
\end{equation*}

Therefore, we know that $X$ correlates well with uncorrputed nodes $S'$ in set $S$. By the basic SDP push-out effect of submatrices, we obtain

\begin{equation*}
	\iprod{\tilde{A}_{S'} - X^*_{S'},X_{S'}} \leq \Paren{2 + \frac{\Deltadelta}{2}} (1-o(1)) n \sqrt{d} \,.
\end{equation*}

Thus, one can deduce that $\iprod{X_{S'},X_{S'}^*} \geq \Omega(n^2)$. Since the entries of $X$ are within $[-1, 1]$, we can further obtain $\iprod{X,X^*} \geq \Omega(n^2)$ as well.
Subsequently, after applying the standard rounding procedure outlined in \cref{lem:rounding}, we obtain an estimator $\hat{x}\in \{\pm 1\}^n$ with a weak recovery guarantee $\iprod{x^*,\hat{x}}^2 \geq \Omega(n^2)$.

\paragraph{Certificates for the submatrix push-out effect}
Even with the \textit{submatrix push-out constraints}, two fundamental challenges remain. \textit{First}, we need to prove the existence of a pair $(S, X)$ that satisfies the submatrix push-out constraints. \textit{Second}, we need to be able to find such a pair efficiently.

With regard to the first challenge, ideally we would like to prove that the set of uncorrupted nodes $S^*$ and the optimizer $X$ of $\SDP(\tilde{A}_{S^*})$ fulfill the submatrix push-out constraints. However, it is difficult prove this: even though we have established that $\iprod{\tilde{A}_{S^*}, X} \geq (2+\Omega(1)) (1-\mu) n\sqrt{d}$, it remains unclear whether $\iprod{\tilde{A}_{S^*} - \tilde{A}_{S'}, X}$ is small for all $S'\subseteq S^*$ of size $(1-2\mu)n$. 

To overcome this barrier, we make the following crucial observation:
\begin{lemma}[Formal statement and 
	proof in \cref{sec:spectral-gr}]\label{lem:spectral-thin-grothendieck}
	Given $S$ of size $(1-\mu)n$, if
	$\Normop{\tilde{A}_{S}}\leq O(\sqrt{d})$, then 
	$\SDP(\tilde{A}_{S}-\tilde{A}_{S'}) \leq O(\mu n \sqrt{d})$
	for all $S'\subseteq S$ of size at least $(1-2\mu) n$.
\end{lemma}

This result suggests us to consider the following program:
\begin{equation}
	\label{eq:algo-sdp}
	\begin{aligned}
		\max_{X,S} \quad & \iprod{\tilde{A}_{S}, X}\\
		\textrm{s.t.} \quad & X \succeq 0 \\
		& X_{ii} = 1 \quad \forall i \in [n] \\
		& \Normop{\tilde{A}_{S}}\leq O(\sqrt{d})
	\end{aligned}
\end{equation}

We begin by establishing the feasibility of the program. Although the spectral norm of $\tilde{A}_{S^*}$ can potentially reach $\polylog(n)$, we can leverage the results of \cite{feige2005spectral} and reduce it to $O(\sqrt{d})$ through the pruning of high-degree nodes. 
The feasibility of the program can then be confirmed by taking $S$ as the set of uncorrputed nodes and have degree at most $O(d)$.
Furthermore, by union bound and the push-out effect established in \cref{eq:basic-sdp-push-out} and \cref{eq:basic-sdp-push-in}, we have $\SDP(\tilde{A}_S)\geq (2+\Deltadelta)\cdot (1-2\mu)n\sqrt{d}$ with high probability. Therefore the objective value of this program is at least $(2+\Deltadelta)\cdot (1-2\mu)n\sqrt{d}$.

The optimizer of program \ref{eq:algo-sdp}, denoted by the pair $(\hat{X}, \hat{S})$, can then be shown to satisfy the submatrix push-out constraints as defined in \cref{def:submatrix-pushout}. 
It follows from our previous argument that the objective value of this program is at least $(2+\Deltadelta)\cdot (1-2\mu)n\sqrt{d}$, which implies $\iprod{\tilde{A}_{\hat{S}},\hat{X}}\geq (2+\Deltadelta)\cdot (1-2\mu)n\sqrt{d}$.
Moreover, the program constraints enforce the bound $\Normop{\tilde{A}_{\hat{S}}}\leq O(\sqrt{d})$.
Together with \cref{lem:spectral-thin-grothendieck}, these implies that $\SDP(\tilde{A}_{\hat{S}}-\tilde{A}_{S'}) \leq O(\mu n \sqrt{d})$ for all $S'\subseteq \hat{S}$ with size at most $(1-2\mu) n$.
When $\mu=o(1)$, it follows that $\SDP(\tilde{A}_{S'})\geq \SDP(\tilde{A}_{\hat{S}})+\SDP(\tilde{A}_{S'}-\tilde{A}_{\hat{S}})\geq (2+\Deltadelta)\cdot (1-o(1))\cdot n\sqrt{d}$ for all $S'\subseteq S_{\max}$ with size at most $(1-2\mu) n$.

As a result, using similar analysis as the previous paragraph, due to the basic SDP push-out effect, the optimizer $\hat{X}$ will now have non-trivial correlation with the ground truth $x^*$, that is $\iprod{\hat{X}, X^*} \geq \Omega(n^2)$.

The last step is to turn this exponential-time algorithm into an efficient one.
Fortunately, the above argument can be captured by the Sum-of-Squares proof system, thereby enabling us to use the Sum-of-Squares relaxation of program \ref{eq:algo-sdp} to obtain an estimator $\hat{X}$ such that $\iprod{\hat{X}, X^*} \geq \Omega(n^2)$.

\paragraph{Node robust algorithms for sparse graphs}
In the degree regime $d\leq \ddelta$, a simpler approach works: \textit{remove high-degree vertices iteratively.}
Although all vertices in the graph could have degree $\omega(1)$ under corruption, our strategy limits the number of removed vertices to $O(\mu n)$ by iteratively removing the highest degree node and one of its random neighbors. 
In this way, in each round, the number of corrupted nodes in the remaining graph is reduced by 
$\Omega(1)$ in expectation, meaning that the algorithm will terminate in $O(\mu n)$ rounds in expectation. As a result, the remaining graph differs from the uncorrupted graph by $O(n)$ edges, which allows us to apply the edge robust algorithm from \cite{ding2022robust}.

\paragraph{Comparison with \cite{liu2022minimax}}
 
In \cite{liu2022minimax}, a weak recovery algorithm  is presented that is robust to $o(n)$ node corruptions when $\epsilon^2 d$ is sufficiently large. 
The algorithm conceptually resembles ours, as it also aims to identify a subgraph with desired properties. However, it falls short of reaching the KS threshold. In particular, when there are no corruptions, their algorithm is reduced to a combination of degree pruning and existing spectral algorithms (\cite{feige2005spectral}), which is not known to provide weak recovery guarantees close to the KS threshold.

\section{Preliminaries and Notations}

\label{sec:preliminary}

In this section, we formally define notations and cover necessary preliminaries that will be used throughout the paper.

\paragraph{Matrix and vector notations} We use $\Ind$ to denote the all 1's vector and $J$ to denote the all 1's matrix, i.e. $J = \Ind \Ind^{\top}$. For a vector $u$, we use $u_i$ to denote its $i$-th entry. For a matrix $M$, we use $M_{ij}$ to denote the $(i,j)$-th entry of $M$, $M \succeq 0$ to denote that $M$ is positive semidefinite, $\Tr(M)$ to denote the trace of $M$, $\Normop{M}$ to denote the spectral/operator norm of $M$ and $\Normf{M}$ to denote the Frobenius norm of $M$.
For two matrices $X$ and $Y$ of the same size, we use $\odot$ to denote the Hadamard product and we define their inner product by
$\iprod{X, Y}
= \sum_{i, j = 0}^n X_{ij} Y_{ij}
= \Tr(X^{\top} Y)$.
Additionally, given a set $S \subseteq [n]$, we use $v_S$ to denote the subvector restricted to the set $S$ and $M_S$ to denote the submatrix of $M$ where we only keep entries in the set $S \times S$, that is $M_S = M \odot (\Ind_{S} \Ind_{S}^{\top})$.

\paragraph{Stochastic block model notations} We use $\delta = \epsilon^2 d - 1$ to denote the distance to the KS threshold, use $A^0$ to denote the adjacency matrix of the uncorrupted graph $G^0$, use $A$ to denote the adjacency matrix of the corrupted graph $G$, use $X^* = x^* (x^*)^{\top}$ to denote the label matrix, use $S^*$ to denote the uncorrupted set of vertices, use $\tilde{A^0} = A^0 - \frac{d}{n} J$ to denote the centered uncorrpted adjacency matrix and use $\tilde{A} = A - \frac{d}{n} J$ to denote the centered corrupted adjacency matrix.

\paragraph{Basic SDP and Grothendieck norm} We define basic SDP and Grothendieck norm as follows

\begin{definition}[Basic SDP]
\label{def:symmetric_grothendieck_norm}
We define basic SDP as follows
\begin{equation}
\label{eq:symmetric_grothendieck_norm1}
    \SDP(M)
    = \max \left\{ \Iprod{M, X} : X \succeq 0, X_{ii} = 1  \forall i \in [n] \right\}
\end{equation}
An equivalent definition (can be easily verified using eigendecomposition of $X$) is
\begin{equation}
\label{eq:symmetric_grothendieck_norm2}
    \SDP(M)
    = \max \left\{ \sum_{i, j = 1}^n M_{ij} \iprod{\sigma_i, \sigma_j} : \sigma_i \sim S^{n-1} \right\}
\end{equation}
where $S^{n-1}$ is the $n$-dimensional unit sphere.
\end{definition}

\begin{definition}[Grothendieck norm]
\label{def:asymmetric_grothendieck_norm}
Let matrix function $P_\Gamma: \R^{n \times n} \rightarrow \R^{2n \times 2n}$ be defined as
\begin{equation*}
P_\Gamma(M)
=
\begin{bmatrix}
0 & M\\
0 & 0
\end{bmatrix}
\end{equation*}
We define Grothendieck norm $\Norm{\cdot}_{Gr}: \R^{n \times n} \rightarrow \R$ as
\begin{equation}
\label{eq:asymmetric_grothendieck_norm1}
    \Norm{M}_{Gr}
    = \max \left\{ \Iprod{P_\Gamma(M), X} : X \succeq 0, X_{ii} = 1 \forall i \in [2n] \right\}
\end{equation}
An equivalent definition (the equivalence can be easily verified using eigendecomposition of $X$) is
\begin{equation}
\label{eq:asymmetric_grothendieck_norm2}
    \Norm{M}_{Gr}
    = \max \left\{ \sum_{i, j = 1}^n M_{ij} \iprod{\sigma_i, \delta_j} : \sigma_i \sim S^{n-1}, \delta_i \sim S^{n-1} \right\}
\end{equation}
where $S^{n-1}$ is the $n$-dimensional unit sphere.
\end{definition}

From \cref{def:symmetric_grothendieck_norm} and \cref{def:asymmetric_grothendieck_norm}, it is easy to get the following inequalities between the basic SDP and Grothendieck norm.

\begin{claim}[Proved in \cref{sec:sdp-gr}]\label{claim-sdp-gr}
Given matrix $M$, we have $\SDP(M) \leq \Norm{M}_{Gr}$.
\end{claim}

\begin{claim}[Proved in \cref{sec:proof_monotonicity}]
\label{claim:monotonicity_sdp}
    Let $M$ be an $n \times n$ matrix whose diagonal entries are 0 and $S \subseteq [n]$ be a subset of indices, we have $\SDP(M_S) \leq \SDP(M)$.
\end{claim}

\paragraph{Grothendieck inequality} The celebrated Grothendieck inequality relates Grothendieck norm and the $\infty \rightarrow 1$ norm.

\begin{definition}[$\infty \rightarrow 1$ norm]
Let us define $\infty \rightarrow 1$ norm $\Norm{\cdot}_{\infty \rightarrow 1}: \R^{n \times n} \rightarrow \R$ as
\begin{equation*}
    \Norm{M}_{\infty \rightarrow 1}
    = \max \left\{ \iprod{x, My} : x, y \in \{\pm 1\}^n \right\}
\end{equation*}
\end{definition}

\begin{theorem}[Grothendieck inequality, see \cite{alon2004approximating}]
\label{theorem:grothendieck_inequality}
Let $M$ be a real matrix of size $n \times n$. We have
\begin{equation*}
    \Norm{M}_{\infty \rightarrow 1} \leq\Norm{M}_{Gr}
    \leq K_G \Norm{M}_{\infty \rightarrow 1}
\end{equation*}
where $K_G$ is a universal constant called the Grothendieck constant.
\end{theorem}

\paragraph{Sum-of-Squares algorithms} In this paper, we employ the Sum-of-Squares hierarchy for both algorithm design and analysis. As a broad category of semidefinite programming algorithms, Sum-of-Squares algorithms provide optimal or state-of-the-art results in algorithmic statistics, as demonstrated by numerous studies, including \cite{hopkins2018mixture,KSS18,pmlr-v65-potechin17a,hopkins2020mean} (see review \cite{barak2014sum,raghavendra2018high}).

\begin{definition}[Sum-of-Squares proof]
Given a set of polynomial inequalities $\mathcal{A} = \{ p_i(x) \geq 0 \}_{i \in [m]}$ in variables $x_1, x_2, \dots, x_n$, a sum-of-squares proof of the inequality $q(x) \geq 0$ is
\begin{equation*}
    q(x) = \sum_{\alpha} a_{\alpha}^2(x) \bar{p}_{\alpha}(x) + \sum_{\beta} b_{\beta}^2(x)
\end{equation*}
where $\{a_\alpha\}$, $\{b_{\beta}\}$ are real polynomials and $\bar{p}_{\alpha}$ is a product of a subset of the polynomials in $\mathcal{A}$. It is a Sum-of-Squares proof of degree-$d$ if 
 all the polynomials in the summation $\Set{a_{\alpha}^2(x) \bar{p}_{\alpha}(x), b_{\beta}^2(x)}$ 
 have degrees no greater than $d$, and we denote this proof as $\mathcal{A}\sststile{d}{x} q(x)\geq 0$.
\end{definition}

\begin{definition}[Sum-of-Squares refutation]
     Given a set of polynomial inequalities $\mathcal{A} = \{ p_i(x) \geq 0 \}_{i \in [m]}$ in variables $x_1, x_2, \dots, x_n$, a sum-of-squares refutation of $\mathcal{A}$ is
\begin{equation*}
    -1 = \sum_{\alpha} a_{\alpha}^2(x) \bar{p}_{\alpha}(x) + \sum_{\beta} b_{\beta}^2(x)
\end{equation*}
where $\{a_\alpha\}$, $\{b_{\beta}\}$ are real polynomials and $\bar{p}_{\alpha}$ is a product of a subset of the polynomials in $\mathcal{A}$. 
It is a Sum-of-Squares proof of degree-$d$ if all the polynomials in the summation $\Set{a_{\alpha}^2(x) \bar{p}_{\alpha}(x), b_{\beta}^2(x)}$ have degrees no greater than $d$.
\end{definition}

\begin{definition}[Pseudo-expectation]
Let $\mathcal{A} = \{ p_i(x) \geq 0 \}_{i \in [m]}$ be a set of polynomial inequalities. A degree-$d$ pseudo-expectation $\tilde{\E}$ for $\mathcal{A}$ is a linear operator that maps polynomials to real numbers such that:
\begin{itemize}
    \item $\tilde{\E}[1] = 1$,
    \item $\tilde{\E}[q^2 (x)] \geq 0$ for every polynomial $q$ with $\deg(q) \leq \frac{d}{2}$,
    \item $\tilde{\E}[q^2 (x) \cdot p_i(x)] \geq 0$ for every polynomial $p_i \in \mathcal{A}$ and every polynomial $q$ with $\deg(q) \leq \frac{d - \deg(p_i)}{2}$.
\end{itemize}
\end{definition}

The following theorem reveals the key relationship between SOS proofs and SOS algorithms.

\begin{theorem}[Informal restatement of \cite{parrilo2000structured, lasserre2001global, barak2014sum}]
\label{theorem:SOS_algorithm}
For a system of polynomial inequalities $\mathcal{A}$ of size $m$, there is an algorithm that either finds a degree-$d$ SOS refutation of $\mathcal{A}$ or finds a degree-$d$ pseudo-expectation for $\mathcal{A}$ in time $\poly(m n^d)$.
\end{theorem}

\section{
Reaching the KS threshold for diverging degree}

\label{sec:SOS_reaching_KS}

In this section, we give an SOS algorithm when average degree $d$ is larger than some constant $\ddelta$
which depends only on $\delta\coloneqq \epsilon^2 d-1$.

We begin by presenting our main technical theorem, which implies \cref{thm:main}.

\begin{theorem}\label{theorem:diverging_KS_threshold_algo}
    Let G be a graph as described in \cref{def:node_corrupted_SBM}, suppose $\delta\geq \Omega(1)$, there exists constants $\ddelta\leq O(1)$ and $\mudelta\geq \Omega(1)$ which only depend on $\delta$, such that when $d\geq \ddelta$ and $\mu\leq \mudelta$, there exists a polynomial-time algorithm (\cref{algo:algo-diverging}) that outputs $\hat{x}\in \Set{\pm 1}^n$ satisfying
    \begin{equation*}
        \E\iprod{\hat{x},x^*}^2\geq \Omega(n^2)\,.
    \end{equation*}
\end{theorem}

Our algorithm is based on the deg-4 SOS relaxation of the following contraint set. Given a node-corrupted graph $G$ generated according to \cref{def:node_corrupted_SBM} and its centered adjacency matrix $\tilde{A}=A-\frac{d}{n}\Ind \Ind^\top$, we consider the following system of polynomial equations in PSD matrix $X$ of size $n \times n$ and $\{0, 1\}$-vector $w$ of size $n$:
\begin{equation}
    \label{eq:SOS}
    \mathcal{A}:=
    \left\{
    \begin{aligned}
    & w_{i}^2 = w_i \quad &\forall i\in [n] &
    \\
    & \sum_i w_{i} = (1 - \mu - \beta)n &
    \\
    & X \succeq 0 &
    \\
    & X_{ii} = 1 \quad &\forall i\in [n] &
    \\
    & \langle \tilde{A} \odot (ww^{\top}), X \rangle \geq (2 + \Delta) (1-\mu-\beta)n \sqrt{d} &
    \\
    & \Normop{\tilde{A} \odot (ww^{\top})} \leq C_s \sqrt{d} &
    \end{aligned}
    \right\}
\end{equation}
Here $\Delta$ and $C_s$ are constants depending on $\delta$, and $\beta$ is the small fraction of high degree nodes we need to prune to get bounded spectral norm according to \cref{corollary:degree_pruning_spectral_norm}.

The outline of our algorithm is given below:

\begin{algorithmbox}[Algorithm reaching KS threshold for diverging degree]
    \label{algo:algo-diverging}
        \mbox{}\\
        \textbf{Input:} Graph $G$ from node-corrupted SBM.
        \begin{enumerate}[1.]
            \item Run deg-4 SOS relaxation of program \ref{eq:SOS} and obtain pseudo-expectation $\tilde{\E}$.
            \item Compute $\hat{X}\coloneqq \tilde{\E} [X]$.
            \item Apply the rounding procedure in \cref{lem:rounding} on $\hat{X}$ to get estimator $\hat{x}$.
        \end{enumerate}
    \end{algorithmbox}

The design and analysis of our SOS algorithm is based on the push-out effect of the basic SDP (\cite{montanari2016semidefinite}) and spectral properties of the adjacency matrix (\cite{feige2005spectral, chin2015stochastic, liu2022minimax}) (see \cref{sec:pushout_sdp} and \cref{appendix:spectral} for more details).
Essentially, we identify a subset of the vertices whose adjacency matrix has large enough basic SDP value and is spectrally bounded. Then, we use the spectral norm bound and the Grothendieck inequality to bound the basic SDP value of the submatrix formed by corrupted vertices in the selected subset.

\subsection{Proof of correctness}

Now, we present the main body of the proof and leave the rounding scheme to \cref{sec:rounding} in the appendix.

\begin{theorem}
\label{theorem:SOS_KS_threshold_algo}
Consider the constraint set in program $\ref{eq:SOS}$, when $\delta\geq \Omega(1)$, there exists functions $\ddelta\leq O(1)$ and $\mudelta\geq \Omega(1)$ which only depend on $\delta$, such that when $d\geq \ddelta$ and $\mu\leq \mudelta$, the following holds with probability at least $1-o(1)$
\begin{equation*}
    \mathcal{A}
    \sststile{4}{X, w}
    \langle X, X^* \rangle
    \geq \Omega(n^2)
\end{equation*}
\end{theorem}

We break down the proof of \cref{theorem:SOS_KS_threshold_algo} into \cref{lem:SOS_feasibility}, \cref{lem:SOS_uncorrupted_subset_correlation} and \cref{lem:SOS_correlation}. For simplicity, let us refer to the set of vertex $i$ with $w_i = 1$ as set $S$, that is $S = \{i \in [n] | w_i = 1\}$.

In \cref{lem:SOS_feasibility}, we prove the feasibility of program \ref{eq:SOS}.

\begin{lemma}[Proof deferred to \cref{sec:proof-sos-feasible}]
\label{lem:SOS_feasibility}
Program \ref{eq:SOS} is feasible with probability $1-o(1)$.
\end{lemma}

Then, in \cref{lem:SOS_uncorrupted_subset_correlation}, we give a deg-4 SOS proof to show that $\iprod{X_{S'},X_{S'}^*}$ is large for some set $S'$ with size at least $(1-2\mu-\beta) n$.

\begin{lemma}
\label{lem:SOS_uncorrupted_subset_correlation}
Consider set $S' = S \cap S^*$, which is the set of uncorrupted vertices in the set $S$ found by the program. For $X$ and $w$ that satisfy the SOS program in \cref{eq:SOS}, we have
\begin{equation*}
\mathcal{A}
\sststile{4}{X, w}
\langle X_{S'}, X^*_{S'} \rangle
\geq \frac{\Delta'(1-\beta)n^2}{\epsilon \sqrt{d}} - O(\frac{\mu n^2}{\epsilon \sqrt{d}})
\end{equation*}
where $\beta$ is the small constant fraction of high degree nodes we need to prune to get bounded spectral norm according to \cref{corollary:degree_pruning_spectral_norm} and $\Delta'=\Delta'(\delta)$ for some value $\Delta'(\delta)$ that only depends on $\delta$.
\end{lemma}

\begin{proof}
We will apply the identity $\langle X_{S'}, X^*_{S'} \rangle = \langle X_{S'}, \frac{n}{\epsilon d} \tilde{A}_{S'} \rangle - \langle X_{S'}, \frac{n}{\epsilon d}\tilde{A}_{S'} - X^*_{S'} \rangle$ and bound the value of $\langle X_{S'}, \frac{n}{\epsilon d}\tilde{A}_{S'} - X^*_{S'} \rangle$ and $\langle X_{S'}, \frac{n}{\epsilon d} \tilde{A}_{S'} \rangle$ separately.

The value of $\langle X_{S'}, \frac{n}{\epsilon d}\tilde{A}_{S'} - X^*_{S'} \rangle$ is easy to bound. From \cref{theorem:pushout_effect_grothendieck} and union bound, we can get that, with probability $1-o(1)$, we have
\begin{equation*}
    \langle X_{S'}, \tilde{A}_{S'} - \frac{\epsilon d}{n} X^*_{S'} \rangle
    \leq \SDP(\tilde{A}_{S'} - \frac{\epsilon d}{n} X^*_{S'})
    \leq (2+\rho) (1-2\mu-\beta) n\sqrt{d}
\end{equation*}

Now, goal is to bound $\langle X_{S'}, \tilde{A}_{S'} \rangle$. We decompose it as follows
\begin{equation}
\label{eq:grothendieck_norm_of_uncorrupted_submatrix}
    \langle X_{S'}, \tilde{A}_{S'} \rangle
    = \langle X_{S}, \tilde{A}_{S} \rangle - \langle X_{S}, \tilde{A}_{S} - \tilde{A}_{S'} \rangle
\end{equation}

From the constraints of \cref{eq:SOS}, we have
\begin{equation}
\label{eq:lem_SOS_eq1}
    \langle X_{S}, \tilde{A}_{S} \rangle \geq (2+\Delta) (1-\mu-\beta)n \sqrt{d}
\end{equation}

To bound the value of $\langle X_{S}, \tilde{A}_{S} - \tilde{A}_{S'} \rangle$, we note that, by constraint $\normop{\tilde{A}_{S}}\leq C_s\sqrt{d}$, we can apply \cref{lem:spectral-thin-grothendieck} to get
\begin{equation}
    \label{eq:lem_SOS_eq2}
    \langle X_{S}, \tilde{A}_{S} - \tilde{A}_{S'} \rangle\leq 
    \SDP(\tilde{A}_{S} - \tilde{A}_{S'})\leq C_s^{\prime} \mu n \sqrt{d}
\end{equation}
for some constant $C_s^{\prime}$.

Plug \cref{eq:lem_SOS_eq1} and \cref{eq:lem_SOS_eq2} into \cref{eq:grothendieck_norm_of_uncorrupted_submatrix}, we get
\begin{equation*}
    \langle X_{S'}, \tilde{A}_{S'} \rangle
    = \langle X_{S}, \tilde{A}_{S} \rangle - \langle X_{S}, \tilde{A}_{S} - \tilde{A}_{S'} \rangle
    \geq (2+\Delta) (1-\mu-\beta)n \sqrt{d} - C_s^{\prime} \mu n \sqrt{d}
\end{equation*}
Now, we can apply the identity $\langle X_{S'}, X^*_{S'} \rangle = \langle X_{S'}, \frac{n}{\epsilon d} \tilde{A}_{S'} \rangle - \langle X_{S'}, \frac{n}{\epsilon d}\tilde{A}_{S'} - X^*_{S'} \rangle$ and get
\begin{align*}
    \langle X_{S'}, X^*_{S'} \rangle
    = & \langle X_{S'}, \frac{n}{\epsilon d} \tilde{A}_{S'} \rangle - \langle X_{S'}, \frac{n}{\epsilon d} \tilde{A}_{S'} - X^*_{S'} \rangle \\
    \geq & \frac{n}{\epsilon d} \Bigparen{(2+\Delta) (1-\mu-\beta)n \sqrt{d} - C_s^{\prime} \mu n \sqrt{d}} - \frac{n}{\epsilon d} (2+\rho) (1-2\mu-\beta) n\sqrt{d} \\
    \geq & \frac{\Delta'(1-\beta)n^2}{\epsilon \sqrt{d}} - O(\frac{\mu n^2}{\epsilon \sqrt{d}})
\end{align*}
\end{proof}

Finally, since $X$ is positive semidefinite and $X_{ii}=1$ for all $i\in [n]$, we can conclude that there is a deg-4 SOS proof to show that correlation $\iprod{X, X^* }$ is large.
\begin{lemma}[Proof deferred to \cref{sec:proof-SOS_correlation}]
\label{lem:SOS_correlation}
For $X$ and $w$ that satisfy the SOS program in \cref{eq:SOS}, we have \begin{equation*}
\mathcal{A}
\sststile{4}{X, w}
\langle X, X^* \rangle
\geq \frac{\Delta'(1-\beta)n^2}{\epsilon \sqrt{d}} - O(\frac{\mu n^2}{\epsilon \sqrt{d}}) -2 \beta n^2
\end{equation*}
where $\beta$ is the small constant fraction of high degree nodes we need to prune to get bounded spectral norm according to \cref{corollary:degree_pruning_spectral_norm} and $\Delta'=\Delta'(\delta)$ for some value $\Delta'(\delta)$ that only depends on $\delta$.
\end{lemma}

Now, we have all the ingredients to prove \cref{theorem:SOS_KS_threshold_algo}

\begin{proof}[Proof of \cref{theorem:SOS_KS_threshold_algo}]
From \cref{lem:SOS_feasibility}, we know that the SOS program in \cref{eq:SOS} is feasible with probability $1-o(1)$. Combine this with \cref{lem:SOS_correlation}, we know that, with probability $1-o(1)$, the SOS program in \cref{eq:SOS} finds $X$ and $w$ such that they satisfy
\begin{equation*}
\mathcal{A}
\sststile{4}{X, w}
\langle X, X^* \rangle \geq \frac{\Delta'(1-\beta)n^2}{\epsilon \sqrt{d}} - O(\frac{\mu n^2}{\epsilon \sqrt{d}}) -2 \beta n^2
\end{equation*}
for some $\beta$ that is the small constant fraction of high degree nodes we need to prune to get bounded spectral norm according to \cref{corollary:degree_pruning_spectral_norm} and $\Delta'=\Delta'(\delta)$ for some value $\Delta'(\delta)$ that only depends on $\delta$.

When $\mu \leq \mudelta$ for some value $\mudelta$ that only depends on $\delta$ and $\beta = \beta(\delta)$ for some value $\beta(\delta)$ that only depends on $\delta$, we have:
\begin{equation*}
\mathcal{A}
\sststile{4}{X, w}
\langle X, X^* \rangle \geq \frac{\Delta'(1-\beta)n^2}{\epsilon \sqrt{d}} - O(\frac{\mu n^2}{\epsilon \sqrt{d}}) -2 \beta n^2
= \theta(\delta) n^2
\end{equation*}
for some $\theta(\delta)$ that only depends on $\delta$. Thus, when $\delta \geq \Omega(1)$, we can get the weak recovery guarantee:
\begin{equation*}
\mathcal{A}
\sststile{4}{X, w}
\langle X, X^* \rangle
\geq \Omega(n^2)
\end{equation*}
\end{proof}

In order to fully establish the validity of \cref{theorem:diverging_KS_threshold_algo}, it remains to apply the standard rounding procedure from \cite{HopkinsS17} on the pseudo-expectation of matrix $X$ (as depicted in \cref{algo:algo-diverging}). We will address this part in \cref{sec:rounding}.

\section*{Acknowledgements}

We would like to thank Afonso S. Bandeira for insightful discussions.

\addcontentsline{toc}{section}{References}
\bibliographystyle{amsalpha}
\bibliography{bib/mathreview,bib/dblp,bib/custom,bib/scholar}

\providecommand{\bysame}{\leavevmode\hbox to3em{\hrulefill}\thinspace}
\providecommand{\MR}{\relax\ifhmode\unskip\space\fi MR }
\providecommand{\MRhref}[2]{%
  \href{http://www.ams.org/mathscinet-getitem?mr=#1}{#2}
}
\providecommand{\href}[2]{#2}
\begin{thebibliography}{ABARS20}

\bibitem[ABARS20]{abbe2020graph}
Emmanuel Abbe, Enric Boix-Adsera, Peter Ralli, and Colin Sandon, \emph{Graph
  powering and spectral robustness}, SIAM Journal on Mathematics of Data
  Science \textbf{2} (2020), no.~1, 132--157.

\bibitem[Abb17]{AbbeReview}
Emmanuel Abbe, \emph{Community detection and stochastic block models: recent
  developments}, The Journal of Machine Learning Research \textbf{18} (2017),
  no.~1, 6446--6531.

\bibitem[AN04]{alon2004approximating}
Noga Alon and Assaf Naor, \emph{Approximating the cut-norm via grothendieck's
  inequality}, Proceedings of the thirty-sixth annual ACM symposium on Theory
  of computing, 2004, pp.~72--80.

\bibitem[Ban18]{DenseSBMThreshold}
Debapratim Banerjee, \emph{Contiguity and non-reconstruction results for
  planted partition models: the dense case}.

\bibitem[BBAP05]{BBPtransition}
Jinho Baik, G{\'e}rard Ben~Arous, and Sandrine P{\'e}ch{\'e}, \emph{Phase
  transition of the largest eigenvalue for nonnull complex sample covariance
  matrices}.

\bibitem[BS14]{barak2014sum}
Boaz Barak and David Steurer, \emph{Sum-of-squares proofs and the quest toward
  optimal algorithms}, arXiv preprint arXiv:1404.5236 (2014).

\bibitem[CRV15]{chin2015stochastic}
Peter Chin, Anup Rao, and Van Vu, \emph{Stochastic block model and community
  detection in sparse graphs: A spectral algorithm with optimal rate of
  recovery}, Conference on Learning Theory, PMLR, 2015, pp.~391--423.

\bibitem[DdNS22]{ding2022robust}
Jingqiu Ding, Tommaso d'Orsi, Rajai Nasser, and David Steurer, \emph{Robust
  recovery for stochastic block models}, 2021 IEEE 62nd Annual Symposium on
  Foundations of Computer Science (FOCS), IEEE, 2022, pp.~387--394.

\bibitem[DKMZ11]{decelle2011asymptotic}
Aurelien Decelle, Florent Krzakala, Cristopher Moore, and Lenka Zdeborov{\'a},
  \emph{Asymptotic analysis of the stochastic block model for modular networks
  and its algorithmic applications}, Physical Review E \textbf{84} (2011),
  no.~6, 066106.

\bibitem[FO05]{feige2005spectral}
Uriel Feige and Eran Ofek, \emph{Spectral techniques applied to sparse random
  graphs}, Random Structures \& Algorithms \textbf{27} (2005), no.~2, 251--275.

\bibitem[HL18]{hopkins2018mixture}
Samuel~B Hopkins and Jerry Li, \emph{Mixture models, robustness, and sum of
  squares proofs}, Proceedings of the 50th Annual ACM SIGACT Symposium on
  Theory of Computing, 2018, pp.~1021--1034.

\bibitem[Hop20]{hopkins2020mean}
Samuel~B Hopkins, \emph{Mean estimation with sub-gaussian rates in polynomial
  time}, The Annals of Statistics \textbf{48} (2020), no.~2, 1193--1213.

\bibitem[HS17]{HopkinsS17}
Samuel~B. Hopkins and David Steurer, \emph{Efficient bayesian estimation from
  few samples: Community detection and related problems}, 58th {IEEE} Annual
  Symposium on Foundations of Computer Science, {FOCS} 2017, Berkeley, CA, USA,
  October 15-17, 2017, {IEEE} Computer Society, 2017.

\bibitem[HY04]{he2004study}
Jun He and Xin Yao, \emph{A study of drift analysis for estimating computation
  time of evolutionary algorithms}, Natural Computing \textbf{3} (2004), no.~1,
  21--35.

\bibitem[KSS18]{KSS18}
Pravesh~K. Kothari, Jacob Steinhardt, and David Steurer, \emph{Robust moment
  estimation and improved clustering via sum of squares}, Proceedings of the
  50th Annual ACM SIGACT Symposium on Theory of Computing, Association for
  Computing Machinery, 2018.

\bibitem[Las01]{lasserre2001global}
Jean~B Lasserre, \emph{Global optimization with polynomials and the problem of
  moments}, SIAM Journal on optimization \textbf{11} (2001), no.~3, 796--817.

\bibitem[LM22]{liu2022minimax}
Allen Liu and Ankur Moitra, \emph{Minimax rates for robust community
  detection}, arXiv preprint arXiv:2207.11903 (2022).

\bibitem[Mas14]{massoulie2014community}
Laurent Massouli{\'e}, \emph{Community detection thresholds and the weak
  ramanujan property}, Proceedings of the forty-sixth annual ACM symposium on
  Theory of computing, 2014, pp.~694--703.

\bibitem[MNS14]{mossel2014belief}
Elchanan Mossel, Joe Neeman, and Allan Sly, \emph{Belief propagation, robust
  reconstruction and optimal recovery of block models}, Conference on Learning
  Theory, PMLR, 2014, pp.~356--370.

\bibitem[MNS15]{mossel2015reconstruction}
\bysame, \emph{Reconstruction and estimation in the planted partition model},
  Probability Theory and Related Fields \textbf{162} (2015), no.~3, 431--461.

\bibitem[MPW16]{Moitra16}
Ankur Moitra, William Perry, and Alexander~S Wein, \emph{How robust are
  reconstruction thresholds for community detection?}, Proceedings of the
  forty-eighth annual ACM symposium on Theory of Computing, 2016, pp.~828--841.

\bibitem[MS16]{montanari2016semidefinite}
Andrea Montanari and Subhabrata Sen, \emph{Semidefinite programs on sparse
  random graphs and their application to community detection}, Proceedings of
  the forty-eighth annual ACM symposium on Theory of Computing, 2016,
  pp.~814--827.

\bibitem[Par00]{parrilo2000structured}
Pablo~A Parrilo, \emph{Structured semidefinite programs and semialgebraic
  geometry methods in robustness and optimization}, Ph.D. thesis, California
  Institute of Technology, 2000.

\bibitem[PS17]{pmlr-v65-potechin17a}
Aaron Potechin and David Steurer, \emph{Exact tensor completion with
  sum-of-squares}, Proceedings of the 2017 Conference on Learning Theory, 2017.

\bibitem[PWBM18]{LowerBoundPCA}
Amelia Perry, Alexander~S Wein, Afonso~S Bandeira, and Ankur Moitra,
  \emph{Optimality and sub-optimality of pca i: Spiked random matrix models},
  The Annals of Statistics \textbf{46} (2018), no.~5, 2416--2451.

\bibitem[RSS18]{raghavendra2018high}
Prasad Raghavendra, Tselil Schramm, and David Steurer, \emph{High dimensional
  estimation via sum-of-squares proofs}, Proceedings of the International
  Congress of Mathematicians: Rio de Janeiro 2018, World Scientific, 2018,
  pp.~3389--3423.

\bibitem[SM19]{stephan2019robustness}
Ludovic Stephan and Laurent Massouli{\'e}, \emph{Robustness of spectral methods
  for community detection}, Conference on Learning Theory, PMLR, 2019,
  pp.~2831--2860.

\end{thebibliography}

\appendix

\section{Appendix organization}
In \cref{sec:rounding}, we complete the proof of \cref{theorem:diverging_KS_threshold_algo} by adding a rounding scheme.
In \cref{sec:node_to_edge}, we give an algorithm that reduces node corruption to edge corruption in the sparse degree regime via degree pruning.
In \cref{sec:robust_Z2}, we use similar techniques to obtain a polynomial-time algorithm for the row/column corrupted $\Z_2$ synchronization problem.
In \cref{sec:lower-bound}, we show that it is impossible to achieve weak recovery when the corruption fraction $\mu$ is larger than $\delta$, therefore, $\mu$ has to be a value that depends on $\delta$.
In \cref{sec:pushout_sdp} and \cref{appendix:spectral}, we present a small summary of previous results on the pushout effect of basic SDP and spectral bounds of degree-pruned adjacency matrix.
In \cref{sec:deferred_proofs}, we present proofs that are deferred to the appendix due to page limit constraint.
\section{Rounding}
\label{sec:rounding}

Now, we complete the proof of \cref{theorem:diverging_KS_threshold_algo} by giving a rounding procedure adapted from Lemma 3.5 of \cite{HopkinsS17} \footnote{Other rounding procedures, such as random Gaussian rounding, also work here.}.

\begin{lemma}[Rounding procedure adapted from Lemma 3.5 of \cite{HopkinsS17}]
\label{lem:rounding}
Let $\theta = \frac{1}{\Norm{X}_F n} \langle X, X^* \rangle$. Let $Y$ be a matrix of minimum Frobenious norm such that $Y \succeq 0$, $\diag Y = 1$ and $\frac{1}{\Norm{X}_F n} \langle Y, X \rangle \geq \theta$. With probability $1-o(1)$, the vector $\hat{x}$ obtained by taking coordiate-wise sign of a Gaussian vector with mean $0$ and covariance $Y$ satisfies
\begin{equation*}
    \E [\langle \hat{x}, x^* \rangle^2 ] \geq \Omega(\theta)^2 n^2
\end{equation*}
\end{lemma}

\begin{proof}
    Apply Lemma 3.5 of \cite{HopkinsS17} by taking $P = X$, $y = x^*$ and $\delta' = \theta$, we can get
    \begin{equation*}
        \E [\langle \hat{x}, x^* \rangle^2 ] \geq \Omega(\theta)^2 n^2
    \end{equation*}
    Notice that, because each entry of $X$ is within $\pm 1$, we have $\Norm{X} \leq n$. Since $\langle X, X^* \rangle \geq \Omega(n^2)$ 
    by \cref{theorem:SOS_KS_threshold_algo}, we have $\theta = \Omega(1)$. Thus, $\hat{x}$ weakly recovers $x^*$.
\end{proof}

Now we finish the proof of \cref{theorem:diverging_KS_threshold_algo}.
\begin{proof}[Proof of \cref{theorem:diverging_KS_threshold_algo}]
 By combining \cref{theorem:SOS_KS_threshold_algo} and \cref{theorem:SOS_algorithm}, 
 we can compute the pseudo-expectation $\tilde{\E}$ for the SOS relaxtion 
  of \cref{eq:SOS} in polynomial time. Let $\hat{X}\coloneqq \tilde{\E} [X]$ in \cref{eq:SOS}. 
 By linearity of pseudo-expectation, we have $\hat{X}\succeq 0$, $\hat{X}_{ii}=1$
  and $\iprod{\hat{X},X^*}\geq \Omega(n^2)$ with probability $1-o(1)$. Now applying rounding procedure 
   in \cref{lem:rounding}, we can then obtain $\hat{x}\in \Set{\pm 1}^n$ such that
   $\E\iprod{\hat{x},x^*}^2\geq \Omega(n^2)$. 
\end{proof}
\section{Reaching KS threshold for constant degree}
\label{sec:node_to_edge}

In this section, we give an algorithm that reduces node corruption to edge corruption when $d < \ddelta$. This allows us to deal with graphs with small average degree.

\subsection{Edge-robust algorithm}

Before introducing our algorithm for the constant degree region, we restate the main theorem of \cite{ding2022robust} here. Their main theorem shows that there exists a polynomial-time algorithm that is robust against $O(\rho n)$ edge perturbations, where $\rho$ is a constant that depends on $\delta$.

\begin{theorem}[Informal restatement of Corollary 5.4 of \cite{ding2022robust}]
\label{theorem:edge_corrupt_algo}
    Given a graph $G \sim \sbm$, suppose $G'$ is an arbitrary graph that differs from $G$ in at most $O(\rho n)$ edges for
    \begin{equation*}
        \rho \leq \bigparen{\frac{1}{\delta} \log \frac{1}{\epsilon}}^{- O(1 / \delta)}
    \end{equation*}
    Then, there exists a polynomial-time algorithm that, given $G'$ and $\delta$, computes an $n$-dimensional unit vector $\hat{x}$ such that
    \begin{equation*}
        \E [\iprod{\hat{x}, x^*}^2] \geq \delta^{O(1)} n
    \end{equation*}
\end{theorem}

\subsection{Degree-pruning based algorithm}

The algorithm is based on degree pruning. The tricky part is that node corruption can arbitrarily increase the degree of uncorrupted vertices to $\mu n$. Therefore, simply pruning high-degree vertices can be quite difficult to analyse.

Our solution is to iteratively remove the highest degree node as well as one of its neighbours that is selected uniformly at random until all vertices have small enough degree. The goal is to make sure that, in each round, we remove $\Omega(1)$ corrupted vertices in expectation.

Notice that, in each round, if the highest degree node is corrupted, then it is good. If the highest degree node is uncorrupted, then we can show, with high probability, the majority of its neighbours are corrupted vertices and we are likely to remove a corrupted vertex if we select one of its neighbours uniformly at random. A key observation is that, this approach allows us to easily bound the total number of removed vertices using a simple and standard Markov Chain drift analysis.

After the degree pruning procedure, we will invoke the edge-robust algorithm from \cite{ding2022robust} that is restated in \cref{theorem:edge_corrupt_algo}.

\begin{algorithmbox}[Algorithm reaching KS threshold for constant degree]
\label{algo:node_to_edge}
	\mbox{}\\
	\textbf{Input:} A node-corrupted stochastic block model $G$.
	\begin{enumerate}[1.]
            \item Set $G' \leftarrow G$
		\item While there exist vertices with degree larger than $C_{\deg}(\mu) d$ in $G'$:
            \begin{itemize}
                \item remove the highest-degree vertex $v$ from $G'$,
                \item remove from $G'$ a neighbour $u$ of $v$ that is selected uniformly at random.
            \end{itemize}
		\item Run edge-robust algorithm from \cref{theorem:edge_corrupt_algo} on the remaining graph $G'$.
        \item Apply the rounding procedure in \cref{lem:rounding} to get estimator $\hat{x}$.
	\end{enumerate}
\end{algorithmbox}

In the following theorem, we will show that \cref{algo:node_to_edge} outputs an estimator $\hat{x}$ that achieves weak recovery.

\begin{theorem}
    \label{theorem:edge_algo_guarantees}
    When $d < \ddelta$ and $\delta = \Omega(1)$, for some $C_{deg}(\mu)$ that only depends on $\mu$, \cref{algo:node_to_edge} outputs a vector $\hat{x} \in \{ \pm 1 \}^n$ such that
    \begin{equation*}
        \E[\iprod{\hat{x}, x^*}^2] \geq \Omega(n^2)
    \end{equation*}
    Moreover, \cref{algo:node_to_edge} runs in polynomial time.
\end{theorem}

\subsection{Proof of correctness}

To prove \cref{theorem:edge_algo_guarantees}, we will use the following two lemmas: \cref{lem:algo_termination} and \cref{lem:algo_corruption_edges}. First, we prove \cref{lem:algo_termination} which says that, with probability 0.99, the pruning step of \cref{algo:node_to_edge} terminates in $O(\mu n)$ rounds.
Then, in \cref{lem:algo_corruption_edges}, we prove that, with probability 0.99, \cref{algo:node_to_edge} produces a graph $G'$ that differs from $G^0$ by at most $O(\rho n)$ edges, such that we can apply \cref{theorem:edge_corrupt_algo} on $G'$ to get an estimator $\hat{x}$ that achieves weak recovery.

\begin{lemma}
    \label{lem:algo_termination}
    With probability at least $0.99$, for some $C_{deg}(\mu)$ that only depends on $\mu$, step 2 of \cref{algo:node_to_edge} terminates in $O(\mu n)$ rounds.
\end{lemma}

\begin{proof}
Let $S$ denote the set of uncorrupted vertices and let $G[S]$ denote the induced subgraph of the uncorrupted vertices. For vertices with degree more than $C_{deg}(\mu) d$ in $G$, we separate them into three cases:
\begin{enumerate}
    \item corrupted vertices,
    \item uncorrupted vertices with degree larger than or equal to $\frac{1}{2} C_{deg}(\mu) d$ in $G[S]$,
    \item uncorrupted vertices with degree smaller than $\frac{1}{2} C_{deg}(\mu) d$ in $G[S]$.
\end{enumerate}
We will prove that, with probability at least $0.99$, all three cases can be eliminated in $O(\mu n)$ rounds. Therefore, with probability 0.99, step 2 of \cref{algo:node_to_edge} terminates in $O(\mu n)$ rounds.

\paragraph{Case 1:} Since there are at most $\mu n$ corrupted vertices, it takes at most $\mu n$ rounds to deal with corrupted vertices with degree more than $C_{deg}(\mu) d$ in $G$.

\paragraph{Case 2:} For uncorrupted vertices with degree larger than or equal to $\frac{1}{2} C_{deg}(\mu) d$ in $G[S]$ and degree more than $C_{deg}(\mu) d$ in $G$, we bound it by the total number of vertices with degree larger than or equal to $\frac{1}{2} C_{deg}(\mu) d$ in $G^0$. By Chernoff Bound, we have that, for each vertex $v$, the probability that $v$ has degree more than $\frac{1}{2} C_{deg}(\mu) d$ in $G^0$ is roughly bounded by
\begin{equation*}
    \Pr[\deg_{G^0}(v) \geq \frac{1}{2} C_{deg}(\mu) d] \leq O(\exp (- \frac{1}{2} C_{deg}(\mu) d))
\end{equation*}
Let $T$ be the set of vertices with degree larger than or equal to $\frac{1}{2} C_{deg}(\mu) d$ in $G^0$. By Markov's inequality, we get that the probability that $T$ contains more than $\mu n$ vertices is roughly bounded by
\begin{equation*}
    \Pr[|T| \geq \mu n] \leq O(\frac{\exp (- \frac{1}{2} C_{deg}(\mu) d)}{\mu})
\end{equation*}
By setting $C_{deg}(\mu)$ to be large enough with respect to $\mu$, we get that, with probability 0.999, there are at most $\mu n$ vertices with degree larger than or equal to $\frac{1}{2} C_{deg}(\mu) d$ in $G^0$. 
Therefore, it takes at most $\mu n$ rounds to remove the vertices that are uncorrupted and has degree more than $\frac{1}{2} C_{deg}(\mu) d$ in $G[S]$.

\paragraph{Case 3:} For uncorrupted vertices that have degree smaller than $\frac{1}{2} C_{deg}(\mu) d$ in $G[S]$ but have degree more than $C_{deg}(\mu) d$ in $G$, the key observation is that more than half of their neighbours are corrupted vertices. Therefore, each time such a node is removed as the highest degree node, with probability more than $1/2$, the algorithm will remove a corrupted node as its random neighbour.

Now, let us only consider the rounds where the highest degree node is in case 3 and let $t$ denote the total number of such rounds.
Let $X_i$ denote the number of corrupted vertices removed after round $i$ and $X_0 = 0$. 
We know that $X_{i+1} = X_i + 1$ with probability more than $1/2$ and $X_{i+1} = X_i$ otherwise. We also know that the process has to terminate when $X_t = \mu n$. Therefore, by the standard Markov Chain drift analysis (see Lemma 1 in \cite{he2004study}), we have:
\begin{equation*}
    \E[t] \leq 2 \mu n
\end{equation*}
and, by Markov inequality, the probability that vertices in case 3 are not eliminated after $1000 \mu n$ rounds where the highest degree node is in case 3 is bounded by
\begin{equation*}
    \Pr[t \geq 1000 \mu n] \leq \frac{2 \mu n}{1000 \mu n} = 0.002
\end{equation*}
Therefore, with probability at least 0.998, vertices in case 3 are eliminated in $1000 n$ rounds.

\paragraph{Conclusion} Taking union bound over the failure probabilities, we get that, with probability at least 0.99, step 2 \cref{algo:node_to_edge} terminates in $1002 \mu n = O(\mu n)$ rounds.
\end{proof}

Notice that \cref{lem:algo_termination} gives us an upper bound on the total number of removed vertices during the pruning step and \cref{algo:node_to_edge} guarantees that $G'$ will have bounded degree after pruning. These two observations allow us to have the following lemma, which says that the difference between $G'$ and $G^0$ is at most $O(\rho n)$ edges, where $O(\rho n)$ is the number of edges that can be tolerated by the edge-robust algorithm in \cref{theorem:edge_corrupt_algo}.

\begin{lemma}
    \label{lem:algo_corruption_edges}
    Let $G^0$ be the uncorrupted graph, with probability 0.99, the remaining graph $G'$ in step 3 of \cref{algo:node_to_edge} differs from $G^0$ by $O(\rho n)$ edges, where $\rho \leq \bigparen{\frac{1}{\delta} \log \frac{1}{\epsilon}}^{- O(1 / \delta)}$ as defined in \cref{theorem:edge_corrupt_algo}.
\end{lemma}

\begin{proof}
    Graph $G'$ differs from $G^0$ by two types of edges:
    \begin{enumerate}
        \item corrupted edges in $G'$,
        \item uncorrupted edges that are removed from pruning.
    \end{enumerate}
    We will bound the two cases separately.
    
    \paragraph{Case 1} \cref{algo:node_to_edge} guarantees that the degree of each vertex in $G'$ is bounded by $C_{\deg}(\mu) d$. Since there are at most $\mu n$ corrupted vertices in $G'$, the maximum number of corrupted edges in $G'$ is $C_{\deg}(\mu) \mu d n$. Since $d < \ddelta$ for some $\ddelta$, we can set $C_{\deg}(\mu)$ to be small enough such that $C_{\deg}(\mu) \mu d n \leq O(\rho n)$.

    \paragraph{Case 2} For case 2, we consider two types of vertices that are removed in \cref{algo:node_to_edge}:
    \begin{itemize}
        \item vertices with degree smaller than or equal to $C_{\deg}(\mu) d$ in $G^0$,
        \item vertices with degree larger than $C_{\deg}(\mu) d$ in $G^0$.
    \end{itemize}
    
    For the first type of vertices, we observe that, with probability $0.99$, \cref{algo:node_to_edge} terminates in $O(\mu n)$ rounds by \cref{lem:algo_termination}. Therefore, with probability $0.99$, there can be at most $O(\mu n)$ vertices in this case. Hence, the number of uncorrupted edges that are removed from pruning the first type of vertices can be bounded by $O(C_{\deg}(\mu) \mu d n)$. 
    Similar to case 1, we can set $C_{\deg}(\mu)$ properly such that $O(C_{\deg}(\mu) \mu d n) \leq O(\rho n)$.

    For the second type of vertices, we know that, for each vertex $v$, the probability that $v$ has degree larger than or equal to $t$ is bounded by
    \begin{equation*}
        \Pr[\deg_{G^0}(v) \geq t] \leq O(\exp(- t))
    \end{equation*}
    Let $X_t$ denote the number of vertices with degree $t$ in $G^0$. We have
    \begin{equation*}
        \E[X_t] \leq \Pr[\deg_{G^0}(v) \geq t] \cdot n \leq O(\exp(- t) n)
    \end{equation*}
    Therefore, for some properly selected $C_{\deg}(\mu)$, the expected total number of edges from vertices with degree larger than $C_{\deg}(\mu) d$ in $G^0$ can be bounded by
    \begin{equation*}
        \E\Brac{\sum_{t = C_{\deg}(\mu) d}^{\infty} X_t t}
        = \sum_{t = C_{\deg}(\mu) d}^{\infty} \E[X_t] t
        \leq O \Paren{\sum_{t = C_{\deg}(\mu) d}^{\infty} \exp(- t) t n}
    \end{equation*}
    By setting $C_{\deg}(\mu)$ to be a large enough value, we have 
    $$\E\Brac{\sum_{t = C_{\deg}(\mu) d}^{\infty} X_t t} \leq A \rho n$$ 
    for some universal constant $A$. By Markov inequality, with probability $0.99$, the number of uncorrupted edges that are removed from pruning second type of vertices can be bounded by $O(\rho n)$.
    
    \paragraph{Conclusion} Take union bound over failure probabilities, we get that, with probability $0.98$, $G'$ differs from $G^0$ by $O(\rho n)$ edges.
\end{proof}

Now, we prove \cref{theorem:edge_algo_guarantees} using  \cref{lem:algo_termination} and \cref{lem:algo_corruption_edges}.

\begin{proof}[Proof of \cref{theorem:edge_algo_guarantees}]
    First, we prove recovery guarantees of \cref{algo:node_to_edge}. From \cref{lem:algo_corruption_edges}, we know that, with probability 0.98, $G'$ differs from $G^0$ by $O(\rho n)$ edges. Combine the guarantees of \cref{theorem:edge_corrupt_algo} and the rounding procedure in \cref{lem:rounding}, step 3 will output an estimator $\hat{x} \in \{ \pm 1 \}^n$ such that
    \begin{equation*}
        \E[\iprod{\hat{x}, x^*}^2]
        \geq 0.98 \cdot \Omega(n^2)
        = \Omega(n^2)
    \end{equation*}

    Now, we prove the time complexity of \cref{algo:node_to_edge}. For step 2 of the algorithm, each round takes at most $O(n^2)$ time and there can be at most $n$ rounds. Therefore, step 2 takes at most $O(n^3)$ time. For step 3, the edge-robust algorithm from \cref{theorem:edge_corrupt_algo} takes polynomial time. For step 4, the rounding procedure from \cref{lem:rounding} takes polynomial time. Therefore, \cref{algo:node_to_edge} runs in polynomial time.
\end{proof}

\section{Lower bound on the corruption fraction}\label{sec:lower-bound}

As stated in \cref{thm:main}, our algorithm is robust against $\Omega_{\delta}(1)$ fraction of corrupted nodes. 
One might wonder whether we can remove the dependency on $\delta=\epsilon^2 d-1$, and find an algorithm robust against $\Omega(1)$ fraction of corrupted nodes (e.g $0.001$ fraction of corrupted nodes).
The following claim shows that this is impossible.

\begin{claim}
Let $n>1,d > 1, \eps \in (0,1)$ and label vector $x^*\in \Set{\pm 1}^n$.
Let $\delta\coloneqq \eps^2 d-1$ and suppose $\delta \geq \Omega(1)$. 
For $G^0\sim \SBM_{d,\epsilon}(x^*)$ and $\mu\geq \delta$, if an adversary removes $\mu n$ vertices uniformly at random from $G^0$ to obtain the graph $G$ that we observe, then it is information theoretically impossible to achieve weak recovery given $G$, i.e. for any estimator $\hat{x}(G)\in \{\pm 1\}^n$, we have
\begin{equation*}
    \E\iprod{\hat{x}(G),x^*}^2\leq o(n^2)\,.
\end{equation*}
\end{claim}

\begin{proof}
Let us denote the set of remaining vertices as $R$.
Note that the remaining graph follows distribution $\SBM_{d',\epsilon'}(x_R^*)$, where $d'=(1-\mu)\cdot d$ and $\epsilon'=\epsilon$. When $\mu\geq \delta$, we have $\epsilon^{\prime 2} d'\leq (1-\delta)^2\cdot (1+\delta)\leq 1$.
According to \cite{mossel2015reconstruction}, it is information theoretically impossible to achieve weak recovery when $\epsilon^{\prime 2} d'\leq 1$.
Thus, it is information theoretically impossible to recover $x^*$.
\end{proof}

\section{Push-out effect of basic SDP}
\label{sec:pushout_sdp}

In this section, we present \cref{theorem:pushout_effect_grothendieck} that captures the push-out effect of the basic SDP value of uncorrupted stochastic block model. This theorem is based on Theorem 5 and Theorem 8 of \cite{montanari2016semidefinite} and is stated in a way that is easier for us to use in our analysis. It is intensively used in \cref{sec:SOS_reaching_KS}, where we prove weak recovery guarantees of our SOS algorithm.

\begin{theorem}[Restatement of Theorem 5 and Theorem 8 of \cite{montanari2016semidefinite}]\label{theorem:pushout_effect_grothendieck}
    Given a graph $G \sim \sbm$, there exists $C = C(\delta)$ and $\ddelta$ that only depend on $\delta = \epsilon^2 d - 1$ such that when $d \geq d_{\delta}$, with probability at least $1 - C e^{-n / C}$, we have
    \begin{equation*}
        \SDP(\tilde{A})
        \geq (2 + \Delta) n \sqrt{d}
    \end{equation*}
    and,
    \begin{equation*}
        \SDP(\tilde{A} - \frac{\epsilon d}{n} X^*)
        \leq (2 + \rho) n \sqrt{d}
    \end{equation*}
    where $\rho = \frac{C \log d}{d^{1/10}}$ and $\Delta=\Delta(\delta)$ for some value $\Delta(\delta)$ that only depends on $\delta$.
\end{theorem}
\section{Spectral bound of degree-pruned submatrix}
\label{appendix:spectral}

In this section, we use the following result to show that we can prune out a small fraction of the high degree vertices to get a spectrally bounded submatrix of the centered adjancency matrix.

\begin{theorem}[Restatement of 
\cite{feige2005spectral, chin2015stochastic, liu2022minimax}]
\label{theorem:CRV15}
Suppose $M$ is a random symmetric matrix with zero on the diagonal whose entries above the diagonal are independent with the following distribution
\begin{equation*}
  M_{ij} =
    \begin{cases}
      1 - p_{ij} & \text{w.p. } p_{ij}\\
      p_{ij} & \text{w.p. } 1- p_{ij}
    \end{cases}       
\end{equation*}
Let $\alpha$ be a quantity such that $p_{ij} \leq \frac{\alpha}{n}$ and $M_1$ be the matrix obtained from $M$ by zeroing out all the rows and columns having more than $20 \alpha$ positive entries. Then with probability $1-\frac{1}{n^2}$, we have
\begin{equation*}
    \Normop{M_1} \leq \chi \sqrt{\alpha}
\end{equation*}
for some constant $\chi$.
\end{theorem}

From \cref{theorem:CRV15}, we can get the following spectral bound for degree-pruned adjacency matrix.

\begin{corollary}
\label{corollary:degree_pruning_spectral_norm}
In the setting of \cref{def:SBM}, with probability at least $1-o(1)$, there exists a subset $T \subseteq [n]$ of size at least $(1 - \beta)n$ such that
\begin{equation*}
    \Normop{\Tilde{A}_{T}} \leq C_s \sqrt{d}
\end{equation*}
where $C_s$ is some constant and $\beta = \beta(\delta)$ is a value that only depends on $\delta$.
\end{corollary}

\begin{proof}
For simplicity, let us set $a$ to be  $a = (1 + \epsilon) d$ and set $t = \beta n$. We apply \cref{theorem:CRV15} by setting $\alpha > a$ to be a large enough constant. The probability that there exists more than $\beta n$ vertices with degree at least $20 \alpha$ is at most
\begin{equation*}
    \binom{n}{t} \binom{t n}{10 \alpha t} \Bigparen{\frac{a}{n}}^{10 \alpha t}
    \leq \Bigparen{\frac{en}{t}}^{t} \Bigparen{\frac{en}{10\alpha}}^{10 \alpha t} \Bigparen{\frac{a}{n}}^{10 \alpha t}
    = \Bigparen{\frac{en}{t}}^{t} \Bigparen{\frac{ea}{10\alpha}}^{10 \alpha t}
\end{equation*}
Since $\alpha > a$, we have
\begin{equation*}
    \Bigparen{\frac{en}{t}}^{t} \Bigparen{\frac{ea}{10\alpha}}^{10 \alpha t}
    \leq \Bigparen{\frac{en}{t}}^{t} \Bigparen{\frac{e}{10}}^{10 \alpha t}
    \leq e^{-10 \alpha t + t(\log(n/t)+1)}
\end{equation*}
Plug in $t = \beta n$, we get the failure probability is $e^{-10 \alpha \beta n + \beta n(\log(1/\beta)+1)}$. As long as $-10 \alpha + \log (1/\beta) + 1 < 0$ for some $\alpha$ and $\beta$, the failure probability is $o(1)$. Take union bound with failure probability of \cref{theorem:CRV15}, we get that, with probability $1-o(1)$, we have
\begin{equation*}
    \Normop{\Bigparen{\Tilde{A}- \frac{\epsilon d}{n} X^*}_{T}} \leq \chi \sqrt{\alpha}
\end{equation*}
Since $\alpha > a > d$, we have
\begin{equation*}
    \Normop{\Bigparen{\Tilde{A} - \frac{\epsilon d}{n} X^*}_{T}} \leq C'_s \sqrt{d}
\end{equation*}
for some constant $C'_s$. Notice that $\Normop{(\frac{\epsilon d}{n} X^*)_{T}} \leq \epsilon d$. Apply triangle inequality, we get
\begin{equation*}
    \Normop{\Tilde{A}_{T}} \leq C'_s \sqrt{d} + \epsilon d
\end{equation*}
When $\epsilon \sqrt{d} = O(1)$, we get $\epsilon d = O(\sqrt{d})$. Hence, with probability $1-o(1)$, we have
\begin{equation*}
    \Normop{\Tilde{A}_{T}} \leq C_s \sqrt{d}
\end{equation*}
for some constant $C_s$.
\end{proof}
\section{Robust $\Z_2$ Synchronization}
\label{sec:robust_Z2}

In this section, we give an algorithm to solve the row/column-corrupted $\Z_2$ synchronization problem using techniques from \cref{sec:SOS_reaching_KS}. The idea is similar to the robust SOS algorithm for node-corrupted stochastic block model: we find a subset of the rows/columns such that the submatrix formed by the subset has large enough basic SDP value and bounded spectral norm. Then, we use the spectral norm bound to upper bound the basic SDP value of the submatrix formed by corrupted rows/columns in the selected subset.

\subsection{Phase transition for $\Z_2$ synchronization}

Before introducing our algorithm, we give a small summary of the basic SDP value phase transition for $\Z_2$ synchronization. It is based on Theorem 5 of \cite{montanari2016semidefinite}, where they gave a very clean result for the phase transition of deformed GOE matrices. The phase transition can naturally be extended to the $\Z_2$ synchronization model using a simple argument based on rotational symmetry. The following theorem informally restates Theorem 5 of \cite{montanari2016semidefinite} and provides the result we need for our robust $\Z_2$ synchronization algorithm.

\begin{theorem}[$\Z_2$ synchronization phase transition \cite{montanari2016semidefinite}]
\label{theorem:pushout_effect_Z2}
    Given an uncorrupted $\Z_2$ synchronization matrix $A^0$ that is generated acoodirng to \cref{def:Z2-corruption},
    \begin{itemize}
        \item if $\sigma \in [0, 1]$, then for any $\xi > 0$, we have $\SDP(A^0) \in [(2 - \xi) n^2, (2 + \xi) n^2]$ with probability $1-o(1)$,
        \item if $\sigma > 1$, then there exists $\Delta(\sigma) > 0$ such that $\SDP(A^0) \geq (2 + \Delta(\sigma))n^2$ with probability $1-o(1)$.
    \end{itemize}
\end{theorem}

\subsection{SOS Algorithm}

In this corruption model, the $\Z_2$ synchronization is easier than the stochastic block model in the sense that, with high probability, $A^0$ is already bounded in spectral norm. Therefore, we can omit the pruning step that we did for robust stochastic block model. Let $A$ be the row/column corrupted $Z_2$ synchronization matrix generated according to \cref{def:Z2-corruption}, consider the following system of polynomial equations in PSD matrix X of size $n \times n$ and vector $w$ of size $n$:
\begin{equation}
\label{eq:Z2_SOS}
\mathcal{A}:=
\left\{
\begin{aligned}
& w_{i}^2 = w_i \quad &\forall i\in [n] &
\\
& \sum_i w_{i} = (1 - \mu)n &
\\
& X \succeq 0 &
\\
& X_{ii} = 1 \quad &\forall i\in [n] &
\\
& \langle A \odot (ww^{\top}) , X \rangle \geq (2 + \Delta(\sigma)) (1 - \mu)^2 n^2 &
\\
& \Normop{A \odot (ww^{\top})} \leq (\sigma + \sigma^{-1}) n &
\end{aligned}
\right\}
\end{equation}
where $\Delta(\sigma) > 0$ is value that only depends on $\sigma$.

In \cref{eq:Z2_SOS}, the vector $w$ is a $\{0, 1\}$ vector that finds a subset of the rows/columns whose submatrix behaves like uncorrupted $\Z_2$ synchronization matrix. Essentially, we require the submatrix $A \odot (ww^{\top})$ to have two properties: \textit{(a)} it has large enough basic SDP value and \textit{(b)} it has small enough spectral norm. The matrix $X$ is a PSD matrix that is a solution to the basic SDP, such that the inner product between $X$ and $A \odot (ww^{\top})$ is large enough. This certifies property (a) and also allows us to obtain our estimator $X$. Property (b) is easy to show because there is an SOS certificate for spectral norm.

We will prove the following theorem for the SOS relaxation of \cref{eq:Z2_SOS}, which implies \cref{thm:sos-z2-informal}. It says that, with high probability, there is a deg-4 SOS proof which shows that any $X$ which satisfies \cref{eq:Z2_SOS} has non-trivial correlation with the true labels $X^*$.

\begin{theorem}[SOS proof for robust $\Z_2$ synchronization]
\label{theorem:SOS_Z2}
When $\sigma > 1$ and $\mu \leq \mu^*(\sigma)$ for some value $\mu^*(\sigma)$ that only depends on $\sigma$, with probability at least $1-o(1)$, we have:
\begin{equation*}
\mathcal{A}
\sststile{4}{X, w}
\langle X, X^* \rangle
\geq \Omega(n^2)
\end{equation*}
\end{theorem}

\subsection{Proof of correctness}

To prove \cref{theorem:SOS_Z2}, we need to show two things:
\begin{enumerate}
    \item $\mathcal{A} \sststile{4}{X, w} \langle X, X^* \rangle \geq \Omega(n^2)$ with high probability,
    \item the constraint set \cref{eq:Z2_SOS} is feasible with high probability.
\end{enumerate}
Now, we prove the first property in the following lemma.

\begin{lemma}
    \label{lem:Z2_SOS_guarantee}
    For $X$ and $w$ satisfying \cref{eq:Z2_SOS}, we have, with probability $1-o(1)$,
    \begin{equation*}
    \mathcal{A}
    \sststile{4}{X, w}
    \langle X, X^* \rangle
    \geq \Omega(n^2)
    \end{equation*}
\end{lemma}

\begin{proof}
Let $s \in \{0, 1 \}^n$ be the indicator variable for the set of uncorrupted indices. We will use the following identity to prove the lemma
\begin{equation}
\label{eq:lem_Z2_SOS_guarantee_1}
    \langle X, X^* \rangle
    = \langle X, X^* \odot (w w^{\top}) \odot (s s^{\top}) \rangle
    + \langle X, X^* \odot (J - (w w^{\top}) \odot (s s^{\top})) \rangle
\end{equation}
Notice that we have $\mathcal{A} \sststile{4}{X, w} X_{ij}^2 \leq 1$ for each $(i, j) \in [n] \times [n]$. Therefore, we can get the following bound for the second term of \cref{eq:lem_Z2_SOS_guarantee_1}
\begin{equation}
\label{eq:lem_Z2_SOS_guarantee_3}
\mathcal{A}
\sststile{4}{X, w}
\langle X, X^* \odot (J - (w w^{\top}) \odot (s s^{\top})) \rangle
\geq - 4 \mu n^2
\end{equation}
Now, the goal is to show that $\langle X, X^* \odot (w w^{\top}) \odot (s s^{\top}) \rangle$ is large enough. To do this, we will use the following identity
\begin{equation}
\label{eq:lem_Z2_SOS_guarantee_2}
    \langle X, X^* \odot (w w^{\top}) \odot (s s^{\top}) \rangle
    = \frac{1}{\sigma}(\langle X, A \odot (w w^{\top}) \odot (s s^{\top}) \rangle
    - \langle X, (A - \sigma X^*) \odot (w w^{\top}) \odot (s s^{\top}) \rangle)
\end{equation}
The easy part is to bound $\langle X, (A - \sigma X^*) \odot (w w^{\top}) \odot (s s^{\top}) \rangle$. We know that $(A - \sigma X^*) \odot (w w^{\top}) \odot (s s^{\top}) = (A^0 - \sigma X^*) \odot (w w^{\top}) \odot (s s^{\top})$ since it is restricted to the set of uncorrupted rows/columns. Moreover, from \cref{theorem:pushout_effect_Z2}, we know that, with high probability, $\SDP(A^0 - \sigma X^*) \leq (2 + \xi) n^2$ for any constant $\xi > 0$. Therefore, by monotonicity of the basic SDP from \cref{claim:monotonicity_sdp}, we have:
\begin{align*}
\mathcal{A}
\sststile{4}{X, w}
& \langle X, (A - \sigma X^*) \odot (w w^{\top}) \odot (s s^{\top}) \rangle \\
= & \langle X, (A^0 - \sigma X^*) \odot (w w^{\top}) \odot (s s^{\top}) \rangle \\
\leq & \SDP((A^0 - \sigma X^*) \odot (w w^{\top}) \odot (s s^{\top})) \\
\leq & \SDP(A^0 - \sigma X^*) \\
\leq & (2 + \xi) n^2
\end{align*}
The hard part is to show that $\langle X, A \odot (w w^{\top}) \odot (s s^{\top}) \rangle$ is large enough. We show this via the following identity:
\begin{equation*}
    \langle X, A \odot (w w^{\top}) \odot (s s^{\top}) \rangle
    = \langle X, A \odot (w w^{\top}) \rangle
    - \langle X, A \odot (w w^{\top}) - A \odot (w w^{\top}) \odot (s s^{\top}) \rangle
\end{equation*}
For the first term $\langle X, A \odot (w w^{\top})\rangle$, we can simply apply the program constraint and get:
\begin{equation*}
\mathcal{A}
\sststile{4}{X, w}
\langle X, A \odot (ww^{\top}) \rangle \geq (2 + \Delta(\sigma)) (1 - \mu)^2 n^2
\end{equation*}
For the second term $\langle X, A \odot (w w^{\top}) - A \odot (w w^{\top}) \odot (s s^{\top}) \rangle$, we bound it by the Grothendieck norm of $A \odot (w w^{\top}) - A \odot (w w^{\top}) \odot (s s^{\top})$. Since we have 
$\mathcal{A}
\sststile{4}{X, w}
\Normop{A \odot (ww^{\top})} \leq (\sigma + \sigma^{-1}) n$
from the program constraints, we can apply \cref{lem:spectral-thin-grothendieck} to get
\begin{equation*}
\mathcal{A}
\sststile{4}{X, w}
\langle X,  A \odot (w w^{\top}) - A \odot (w w^{\top}) \odot (s s^{\top}) \rangle \leq O((\sigma + \sigma^{-1}) \mu n^2)
\end{equation*}
Thus, we have:
\begin{equation*}
\mathcal{A}
\sststile{4}{X, w}
\langle X, A \odot (w w^{\top}) \odot (s s^{\top}) \rangle \geq (2 + \Delta(\sigma)) (1 - \mu)^2 n^2 - O((\sigma + \sigma^{-1}) \mu n^2)
\end{equation*}
Plug the two parts into \cref{eq:lem_Z2_SOS_guarantee_2}, we get:
\begin{align}
\label{eq:lem_Z2_SOS_guarantee_4}
\begin{split}
\mathcal{A}
\sststile{4}{X, w}
\langle X, X^* \odot (w w^{\top}) \odot (s s^{\top}) \rangle
&\geq \frac{1}{\sigma} \Paren{ (2 + \Delta(\sigma)) (1 - \mu)^2 n^2 - O((\sigma + \sigma^{-1}) \mu n^2) - (2 + \xi) n^2 } \\
& = \theta(\sigma) n^2
\end{split}
\end{align}
for some $\theta(\sigma)$ that only depends on $\sigma$.

Now, plug \cref{eq:lem_Z2_SOS_guarantee_3} and \cref{eq:lem_Z2_SOS_guarantee_4} into \cref{eq:lem_Z2_SOS_guarantee_1}, we get
\begin{align*}
\begin{split}
\mathcal{A}
\sststile{4}{X, w}
\langle X, X^* \rangle
& = \langle X, X^* \odot (w w^{\top}) \odot (s s^{\top}) \rangle
    + \langle X, X^* \odot (J - (w w^{\top}) \odot (s s^{\top})) \rangle \\
& \geq \theta(\sigma) n^2 - 4 \mu n^2
\end{split}
\end{align*}
When $\mu \leq \mu^*(\sigma)$ for some value $\mu^*(\sigma)$ that only depends on $\sigma$, we have
\begin{equation*}
\begin{split}
\mathcal{A}
\sststile{4}{X, w}
\langle X, X^* \rangle \geq \theta'(\sigma) n^2
\end{split}
\end{equation*}
where $\theta'(\sigma)$ is a value that only depends on $\sigma$. Thus, when $\sigma > 1$, we have, with probability $1-o(1)$,
\begin{equation*}
\mathcal{A}
\sststile{4}{X, w}
\langle X, X^* \rangle
\geq \Omega(n^2)
\end{equation*}
\end{proof}

Now, we are ready to prove \cref{theorem:SOS_Z2}. In the proof, we first prove feasibility of the contraint set in \cref{eq:Z2_SOS}, then use \cref{lem:Z2_SOS_guarantee} to complete the proof.

\begin{proof}[Proof of \cref{theorem:SOS_Z2}]
    The feasibility analysis is similar to the feasibility analysis in \cref{lem:SOS_feasibility}. From \cref{theorem:pushout_effect_Z2} and union bound, we get that the inequality $\langle A \odot (ww^{\top}) , X \rangle \geq (2 + \Delta(\sigma)) (1 - \mu)^2 n^2$ is feasible with probability $1-o(1)$. The inequality $\Normop{A \odot (ww^{\top})} \leq (\sigma + \sigma^{-1}) n$ is feasible with probability $1-o(1)$ due to the famous BBP phase transition and monotonicity of spectral norm. Take union bound over failure probabilities of the two inequalities, we can conclude the feasibility analysis.
    
    From \cref{lem:Z2_SOS_guarantee}, we get that, with probability $1-o(1)$, we have
    $\mathcal{A}
    \sststile{4}{X, w}
    \langle X, X^* \rangle
    \geq \Omega(n^2)$. Therefore, we can take union bound and conclude that, with probability $1-o(1)$, the program finds an $X$ such that
    \begin{equation*}
    \mathcal{A}
    \sststile{4}{X, w}
    \langle X, X^* \rangle
    \geq \Omega(n^2)
    \end{equation*}
\end{proof}

Now we finish the proof of \cref{thm:sos-z2-informal}.
\begin{proof}[Proof of \cref{thm:sos-z2-informal}]
 By combining \cref{theorem:SOS_Z2} and \cref{theorem:SOS_algorithm}, 
 we can compute the pseudo-expectation $\tilde{\E}$ for the SOS relaxtion 
  of \cref{eq:Z2_SOS} in polynomial time. Let $\hat{X}\coloneqq \tilde{\E} [X]$ in \cref{eq:SOS}. 
 By linearity of pseudo-expectation, we have $\hat{X}\succeq 0$, $\hat{X}_{ii}=1$.
  Furthermore, we have $\iprod{\hat{X},X^*}\geq \Omega(n^2)$ with probability $1-o(1)$. Now, applying rounding procedure 
   in \cref{lem:rounding}, we can then obtain $\hat{x}\in \Set{\pm 1}^n$ such that
   $\E\iprod{\hat{x},x^*}^2\geq \Omega(n^2)$. 
\end{proof}
\section{Deferred proofs}
\label{sec:deferred_proofs}

\subsection{Proof of \cref{claim-sdp-gr}}\label{sec:sdp-gr}
\begin{claim}[Restatement of \cref{claim-sdp-gr}]
    Given matrix $M$, we have $\SDP(M) \leq \Norm{M}_{Gr}$.
\end{claim}

\begin{proof}
    If we look at the second definition of the basic SDP in \cref{eq:symmetric_grothendieck_norm2} and the second definition of Grothendieck norm in \cref{eq:asymmetric_grothendieck_norm2}, it is easy to check that the optimizer of \cref{eq:symmetric_grothendieck_norm2} is a solution to \cref{eq:asymmetric_grothendieck_norm2} if we take $\delta_i = \sigma_i$. Hence, we have
    \begin{equation*}
        \SDP(M) \leq \Norm{M}_{Gr}
    \end{equation*}
\end{proof}

\subsection{Proof of \cref{claim:monotonicity_sdp}}
\label{sec:proof_monotonicity}

\begin{claim}[Restatement of \cref{claim:monotonicity_sdp}]

        Let $M$ be an $n \times n$ matrix whose diagonal entries are 0 and $S \subseteq [n]$ be a subset of indices, we have
        \begin{equation*}
            \SDP(M_S) \leq \SDP(M)
        \end{equation*}
    \end{claim}
\begin{proof}
    Let $X$ be the optimizer of $\SDP(M_S)$ and $Z = X_S + \Id_{[n] \setminus s}$. We have
    \begin{align*}
        \SDP(M_S)
        = & \iprod{X, M_S} \\
        = & \iprod{X_S, M}
     \end{align*}
    Since $M$ has zero on diagonals, we have $\iprod{X_S, M} = \iprod{Z, M}$. Notice that $Z \succeq 0$ and $Z_{ii} = 1$ for all $i \in [n]$. Therefore, $Z$ is a solution to the basic SDP, which implies that
    \begin{equation*}
        \iprod{Z, M} \leq \SDP(M)
    \end{equation*}
    Thus, we have
     \begin{align*}
        \SDP(M_S)
        = & \iprod{X_S, M} \\
        = & \iprod{Z, M} \\
        \leq & \SDP(M)
     \end{align*}
\end{proof}

\subsection{Proof of \cref{lem:spectral-thin-grothendieck}}\label{sec:spectral-gr}

\begin{lemma}[Formal statement of \cref{lem:spectral-thin-grothendieck}]
    Let $\tilde{A}\in \mathbb{R}^{n\times n}$ and $S\subset [n]$ be a set of size $(1-\mu)n$.
    Suppose $\Normop{\tilde{A}_{S}}\leq C_s\sqrt{d}$ for some constant $C_s$, then for all $S'\subseteq S$ with size at least $(1-2\mu) n$,
    there is a deg-4 SOS proof that
    \begin{equation*}
    \SDP(\tilde{A}_{S}-\tilde{A}_{S'})
    \leq O(\mu n \sqrt{d}) \,.
    \end{equation*}
\end{lemma}

\begin{proof}
Consider an arbitrary matrix $X \in \R^{n\times n}$ such that $X_{ii}=1$ for $i\in [n]$ and $X\succeq 0$. It follows that

\begin{align*}
    \iprod{\tilde{A}_{S} - \tilde{A}_{S'}, X}
    & =  \iprod{\tilde{A}_{S} - \tilde{A}_{S'}, X_S - X_{S'}} \\
    & \leq \normop{\tilde{A}_{S} - \tilde{A}_{S'}} \Tr (X_S - X_{S'}) \\
    & \leq (\normop{\tilde{A}_{S}} + \normop{\tilde{A}_{S'}}) \cdot \mu n \\
    & \leq 2 \normop{\tilde{A}_{S}} \cdot \mu n \\
    & \leq 2 C_s \sqrt{d} \cdot \mu n \\
    & = O(\mu n \sqrt{d})
\end{align*}

Notice that, every step of the proof can be made to be deg-4 SOS. Hence, the proof is deg-4 SOS.
\end{proof}

\subsection{Proof of \cref{lem:SOS_feasibility}}\label{sec:proof-sos-feasible}

\begin{lemma}[Restatement of \cref{lem:SOS_feasibility}]
    The SOS program in \cref{eq:SOS} is feasible with probability $1-o(1)$.
\end{lemma}

\begin{proof}
    From \cref{corollary:degree_pruning_spectral_norm}, we know that, with probability $1-o(1)$, there exists a submatrix $\tilde{A}_T$ of size $(1-\beta)n$ whose spectral norm is bounded by $C_s \sqrt{d}$. By monotonicity of spectral norm, the spectral norm of all submatrices of size $(1-\mu - \beta)n$ of $\tilde{A}_T$ are bounded by $C_s \sqrt{d}$. Therefore, if we consider the set $S = T \cap S^*$, it satisfies the spectral constraint with probability $1-o(1)$.
    
    Now, we need to show that, with probability $1-o(1)$, the matrix $\tilde{A}_S$ with $S = T \cap S^*$ has large enough basic SDP value. Apply \cref{theorem:pushout_effect_grothendieck}, we get that with probability at least $1 - C e^{-(1 - \mu - \beta)n/C}$, a stochastic block model of size $(1 - \mu - \beta)n$ has basic SDP value larger than or equal to $(2+\Delta) (1-\mu-\beta)n \sqrt{d}$. Consider all submatrices of size $(1-\mu - \beta)n$ of $\tilde{A}$ and take union bound, the failure probability is
    \begin{align*}
        \binom{n}{(1 - \mu - \beta)n} C e^{-(1 - \mu - \beta)n/C}
        \leq & C \Bigparen{\frac{en}{(\mu + \beta)n}}^{(\mu + \beta)n} e^{-(1 - \mu - \beta)n/C} \\
        = & C e^{(\mu + \beta)n (\log(1/(\mu + \beta)) + 1) - (1-\mu - \beta)n/C}
    \end{align*}
    When $\mu \leq \mudelta$ for some value $\mudelta$ that only depends on $\delta$ and $\beta \ll 1/C$, the failure probability is $o(1)$. Therefore, with probability $1-o(1)$, for the uncorrputed stochastic block model, the basic SDP value of all its submatrices of size $(1-\mu - \beta)n$ is larger than or equal to $(2 + \Delta) (1-\mu-\beta)n \sqrt{d}$, which include the submatrix defined by the set $S = T \cap S^*$.
    
    Hence, with probability $1-o(1)$, there exists a subset $S = T \cap S^*$ of size $(1-\mu-\beta)n$ such that $\tilde{A}_S$ has basic SDP value larger than or equal to $(2 + \Delta) (1-\mu-\beta)n \sqrt{d}$ and has spectral norm less than or equal to $C_s \sqrt{d}$.
    
    For the value of $X$, we can simply take the optimizer of the basic SDP for $\tilde{A}_S$. This concludes the feasibility analysis of the program.
    \end{proof}

\subsection{Proof of \cref{lem:SOS_correlation}}
\label{sec:proof-SOS_correlation}

\begin{lemma}[Restatement of \cref{lem:SOS_correlation}]
    For $X$ and $w$ that satisfy the SOS program in \cref{eq:SOS}, we have \begin{equation*}
    \mathcal{A}
    \sststile{4}{X, w}
    \langle X, X^* \rangle
    \geq \frac{\Delta'(1-\beta)n^2}{\epsilon \sqrt{d}} - O( \frac{\mu n^2}{\epsilon \sqrt{d}}) -2 \beta n^2
    \end{equation*}
    where $\beta$ is the small constant fraction of high degree nodes we need to prune to get bounded spectral norm according to \cref{corollary:degree_pruning_spectral_norm} and $\Delta'=\Delta'(\delta)$ for some value $\Delta'(\delta)$ that only depends on $\delta$.
\end{lemma}
\begin{proof}
    We decompose $\langle X, X^* \rangle$ into $\langle X, X^* \rangle = \langle X_{S'}, X^*_{S'} \rangle + \langle X - X_{S'}, X^* \rangle$. For $\langle X_{S'}, X^*_{S'} \rangle$, we can apply \cref{lem:SOS_uncorrupted_subset_correlation} and get
    \begin{equation*}
    \mathcal{A}
    \sststile{4}{X, w}
    \langle X_{S'}, X^*_{S'} \rangle
    \geq \frac{\Delta'(1-\beta)n^2}{\epsilon \sqrt{d}} - O(\frac{\mu n^2}{\epsilon \sqrt{d}})
    \end{equation*}
    
    Now, we consider $\langle X - X_{S'}, X^* \rangle$. Notice that, since $X$ is positive semidefinite whose diagonals are 1's, all its entries are within $[-1, 1]$. This is because all principle submatrices of a positive semidefinite matrix are positive semidefinite. If we consider the principle submatrix formed by $X_{ii}$, $X_{ij}$, $X_{ji}$ and $X_{jj}$, its determinant is non-negative. Hence, $X_{ij}^2 \leq X_{ii} X_{jj} = 1$. Since there can be at most $(2\mu + \beta)n$ vertices that are not in $S'$, $\langle X - X_{S'}, X^* \rangle$ is a summation of at most $2 (2\mu + \beta) n^2$ entries whose absolute values are less than or equal to 1. Therefore, we have
    \begin{equation*}
        |\langle X - X_{S'}, X^* \rangle| \leq 2 (2\mu+\beta) n^2
    \end{equation*}
    
    Combine the bounds on $\langle X_{S'}, X^*_{S'} \rangle$ and $\langle X - X_{S'}, X^* \rangle$, we have
    \begin{align*}
        \mathcal{A}
        \sststile{4}{X, w}
        \langle X, X^* \rangle
        = & \langle X_{S'}, X^*_{S'} \rangle + \langle X - X_{S'}, X^* \rangle \\
        \geq & \frac{\Delta'(1-\beta)n^2}{\epsilon \sqrt{d}} - O(\frac{\mu n^2}{\epsilon \sqrt{d}}) - 2 (2\mu+\beta) n^2 \\
        \geq & \frac{\Delta'(1-\beta)n^2}{\epsilon \sqrt{d}} - O(\frac{\mu n^2}{\epsilon \sqrt{d}}) -2 \beta n^2
    \end{align*}
    which finishes the proof.
    \end{proof}

\end{document}